\documentclass[10pt,journal,compsoc,twocolumn]{IEEEtran}

\usepackage{amsfonts}
\usepackage{amsmath}
\usepackage{amsthm}
\usepackage{graphicx}
\usepackage{subfigure}
\usepackage{multirow}
\usepackage{multicol}
\usepackage{graphicx}
\usepackage{epstopdf}
\usepackage[justification=centering]{caption}
\usepackage{color}
\usepackage{cite}

\newtheorem{theorem}{Theorem}
\newtheorem{corollary}{Corollary}
\newtheorem{lemma}{Lemma}

\newtheorem{proposition}{Proposition}

\newtheorem{definition}{Definition}

\begin{document}

\title{Deep Convolutional Neural Networks with Zero-Padding: Feature Extraction and Learning}

\author{Zhi Han, Baichen Liu,
        Shao-Bo Lin, ~and Ding-Xuan Zhou
\IEEEcompsocitemizethanks{
\IEEEcompsocthanksitem Z. Han and B. Liu are with the State Key Laboratory of Robotics, Shenyang Institute of Automation, Chinese Academy of Sciences, Shenyang, China and Institutes for Robotics and Intelligent Manufacturing, Chinese Academy of Sciences, Shenyang, China.

\IEEEcompsocthanksitem S.B. Lin is with the  Center of Intelligent Decision-Making and Machine Learning, School of Management, Xi'an Jiaotong University, Xi'an, China.

\IEEEcompsocthanksitem D.X. Zhou  is with the School of Mathematics and Statistics, University of Sydney, Sydney NSW 2006, Australia.

\IEEEcompsocthanksitem Corresponding author: S. B. Lin (sblin1983@gmail.com)}}


\IEEEtitleabstractindextext{%
\begin{abstract}
This paper studies the performance of deep convolutional neural networks (DCNNs) with zero-padding in feature extraction and learning. After verifying the roles of zero-padding in enabling   translation-equivalence, and  pooling in its translation-invariance driven nature, we show that  with similar number of free parameters,   any deep fully connected networks (DFCNs) can be represented by DCNNs with zero-padding. This demonstrates that DCNNs with zero-padding is essentially better than DFCNs in feature extraction. Consequently, we derive universal consistency of DCNNs with zero-padding and show its translation-invariance in the learning process. All our theoretical results are verified by numerical experiments including both toy simulations and real-data running.
\end{abstract}


\begin{IEEEkeywords}
Deep learning, deep convolutional neural network, zero-padding, pooling, learning theory
\end{IEEEkeywords}}

\maketitle

\IEEEdisplaynontitleabstractindextext

\IEEEpeerreviewmaketitle

\IEEEraisesectionheading{\section{Introduction}\label{Sec.Introduction}}

\IEEEPARstart{I}n the era of big data, machine learning, especially deep learning, has received  unprecedented success in numerous  application regions including computer vision \cite{cipolla2013machine}, management science \cite{pallathadka2021impact},  finance \cite{dixon2020machine}, economics \cite{athey2018impact}, games \cite{bowling2006machine} and so on.  To reveal   the mystery   behind the success is the recent focus and will be an
eternal  task of machine learning, requiring not only  to understand the running mechanism of specific learning schemes, but also to provide solid theoretical verifications. As shown in Figure \ref{steps}, a machine learning scheme can be regarded as a two-step strategy that searches suitable feature mappings to obtain a sequence of features at first and then finds  appropriate linear combinations of the obtained  features for the learning purpose. In this way, the main difference of learning schemes lies in the different feature mappings.
For example, the kernel approach \cite{shawe2004kernel} utilizes the kernel-based mapping to extract features and deep learning \cite{goodfellow2016deep} employs deep neural networks (deep nets for short) for feature representation.

\begin{figure}[htbp]
\centerline{\includegraphics[width=0.5\textwidth]{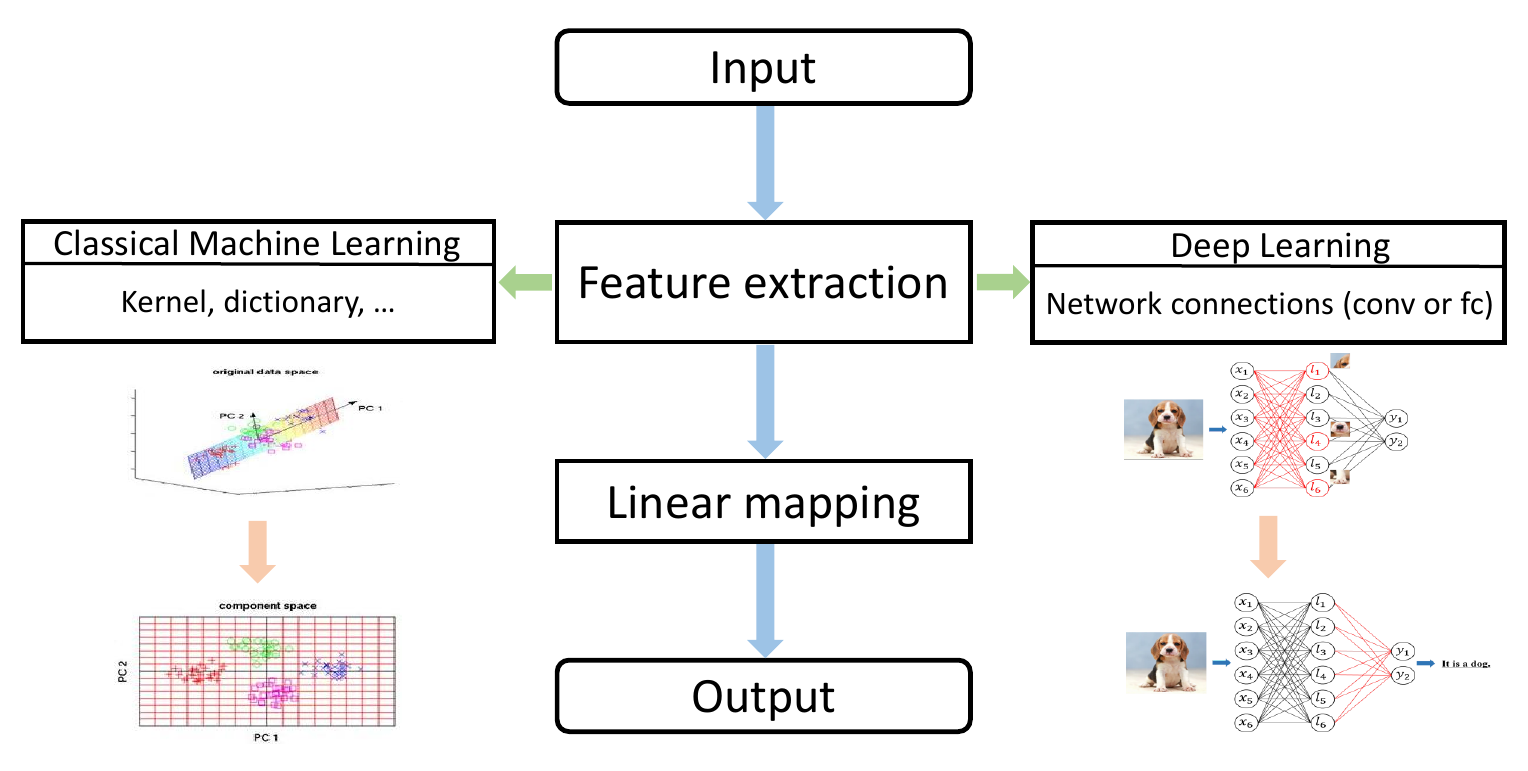}}
\caption{The steps of feature extraction of classical machine learning and deep learning.}
\label{steps}
\end{figure}

The great success of deep learning \cite{deng2014deep,lecun2015deep,goodfellow2016deep} implies that deep nets are excellent feature extractors in representing the translation invariance \cite{kayhan2020translation}, rotation invariance \cite{zhang2019rotation}, calibration invariance \cite{chatzidakis2019towards}, spareness \cite{bengio2009learning},    manifold structure\cite{brahma2015deep}, among others \cite{bengio2013representation}.
Furthermore, avid   research activities in deep learning theory   proved that, equipped with suitable  structures, deep nets outperform the classical shallow neural networks (shallow nets for short) in capturing the smoothness \cite{yarotsky2017error}, positioning the input \cite{chui2020realization},  realizing the sparseness in  frequency  and spatial   domains \cite{schwab2019deep,lin2018generalization}, reflecting the rotation-invariance \cite{chui2019deep}, grasping the composite-structure \cite{mhaskar2016deep}, group-structure \cite{han2020depth}, manifold-structure \cite{shaham2018provable} and hierarchy-structure \cite{kohler2016nonparametric}. These encouraging theoretical assertions together with the
notable success in application  seem that the running mechanism of deep learning is put on  the order of the day.

The problem is, however,  that there are three challenged gaps between the established theoretical results and desired running mechanisms of deep learning.  At first, the structures of deep nets in theoretical analysis and applications for the same purpose are totally different. For example, to embody the rotation invariance, practicioners  focus  on tailoring the network to obtain a structured deep nets to automatically extracting the feature  \cite{zhang2019rotation} while theoretical analysis devotes to proving the existence of a deep fully connected neural network (DFCN) via tuning the weights. Then, as the product  of full matrix frequently does not obey to the commutative law, DFCN
 requires  strict orders of the extracted features, which consequently yields that
the features extracted by DFCNs cannot be  combined directly to explain the success of deep learning in practice. At last,
practicioners are willing to encode the a-priori knowledge into the training process via setting suitable network structure which is beyond the scope of existing theoretical analysis.
The
 above three  gaps between existing theoretical analysis and application requirements, highlighted in Figure \ref{gaps}, significantly dampens the spirits  of both  practicioner and theoretical analysts, making them drive  in totally different directions in understanding deep learning.

\begin{figure}[htbp]
\centerline{\includegraphics[width=0.48\textwidth]{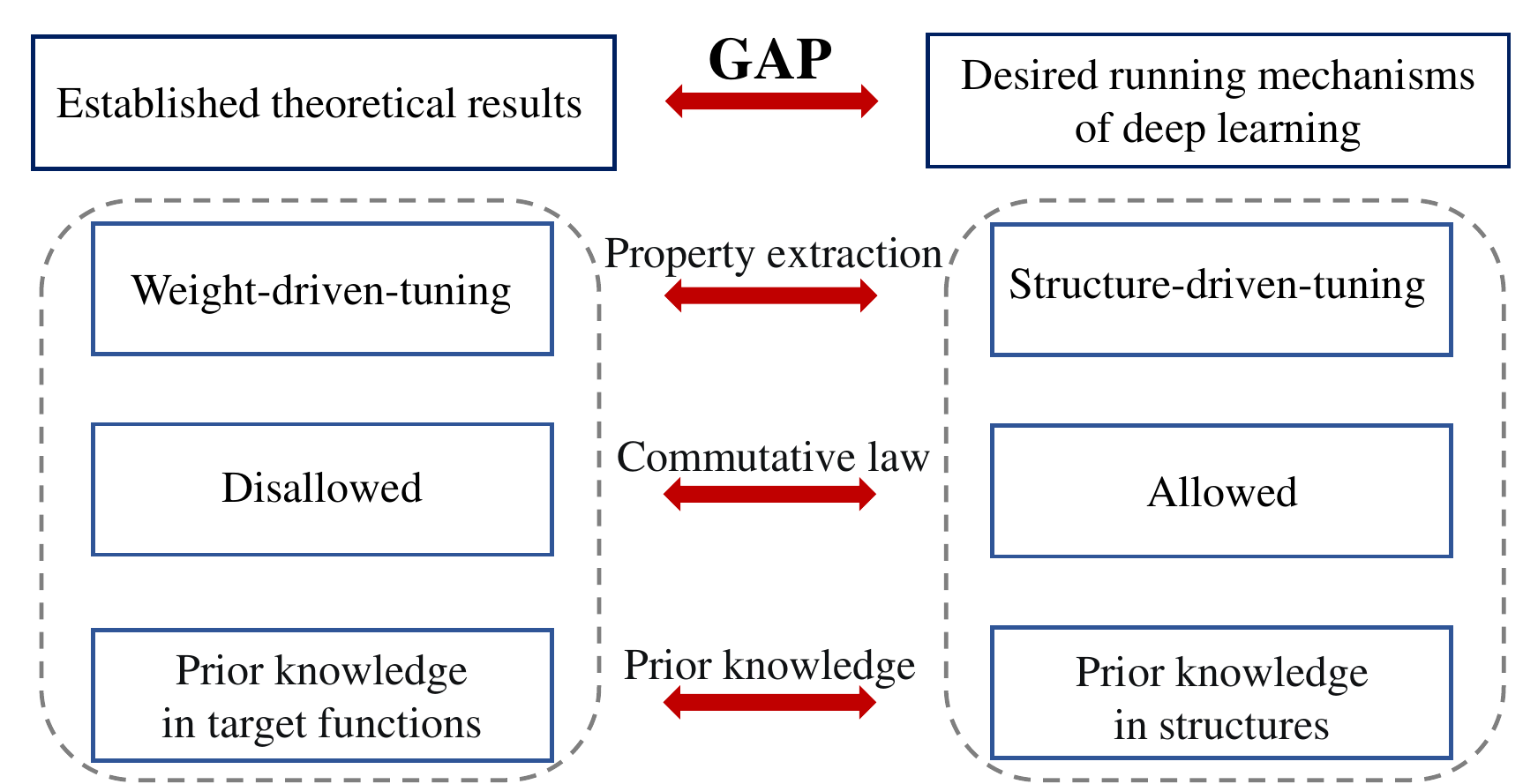}}
\caption{Three challenged gaps between the established theoretical results and desired running mechanisms of deep learning.}
\label{gaps}
\end{figure}

Noting these gaps, some theoretical analysis has been  carried out on analyzing the learning performance of structured deep nets with the intuition that there exist  some features that  can be extracted by the structure of deep nets. Typical examples includes \cite{chui2019deep} for deep nets with tree structures, \cite{petersen2020equivalence} for deep convolutional neural networks (DCNN) with multi-channel, \cite{oono2019approximation} for  DCNN with resnet-type structures, and   \cite{zhou2018deep,zhou2020universality} for deep convolutional neural networks (DCNN) with zero-padding. However, these results seem not to provide sufficient advantages of structured deep neural networks since they neither present   solid theoretical explanations on   why structured deep nets  outperform others, nor give theoretical guidances on which features can be extracted  by specific network structures. This motivates our study in this paper to show why and when DCNN performs better than DFCN and how to tailor the structure (zeor-padding, depth, filter length and pooling scheme) of DCNN for a specific learning task.

\begin{figure}[htbp]
\centerline{\includegraphics[width=0.5\textwidth]{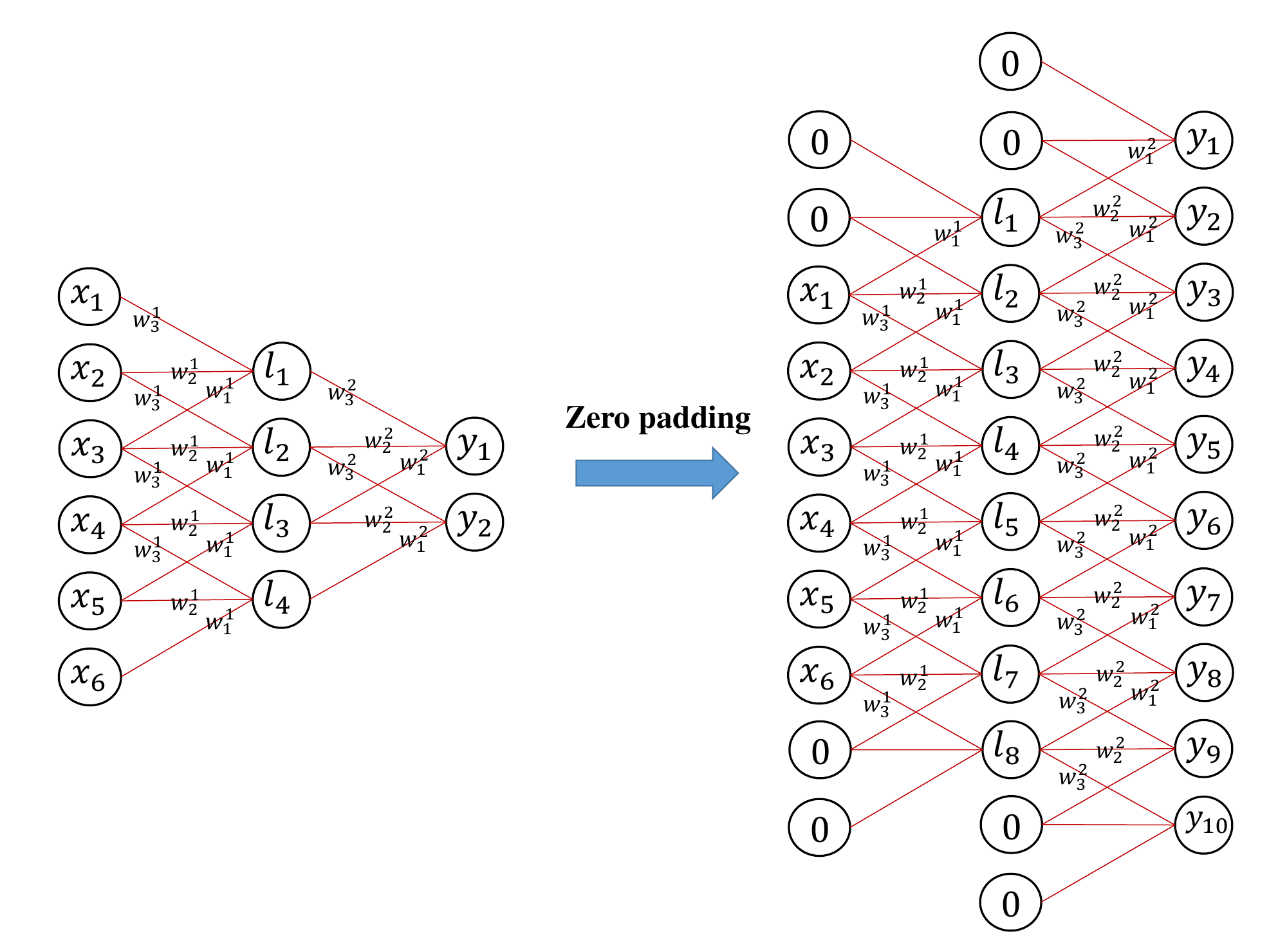}}
\caption{Zero-padding and network structure.}
\label{fig:zero-padding}
\end{figure}

We  study  the performance of DCNN with one-dimensional and one-channel convolution in
feature extraction  and learning.
As shown in Figure \ref{fig:zero-padding}, if there is not any zero-padding imposed to the  convolution structure, DCNN then
behaves a contracting nature (we call such DCNN as  cDCNN for short) in the sense that the width decreases with respect to the depth,    prohibiting its  universality in feature extraction \cite{hanin2017approximating,lin2022universal}.
Therefore, without the help of other structures such as fully-connected layers and zero-padding, there are numerous   features that cannot be extracted by  cDCNN. Noticing that additional fully-connected layer may destroy the convolutional structure of cDCNN, we are interested in conducting zero-padding as Figure \ref{fig:zero-padding}  and making the network possess an expansive nature. We call DCNN with zero-padding like Figure \ref{fig:zero-padding} as eDCNN and study the performance of eDCNN in feature extraction and learning via  considering   the following three problems:

$\diamond$ (P1): What is the role of zero-padding in eDCNN?

$\diamond$ (P2): How to specify the  pooling scheme to improve the performance of eDCNNs?

$\diamond$ (P3): Why and when are DCNNs better than widely studied deep fully connected networks (DFCNs)?

As shown in Figure \ref{fig:zero-padding}, zero-padding enlarges the  size of the  extracted features to guarantee the universality, and therefore plays a crucial role in DCNN.   Problem (P1) indeed refers to how to explore the important role of  zero-padding in enhancing the   representation and learning performances of cDCNN.  It is well known that pooling drives the converse direction as zero-padding  to shrink the size via designing suitable sub-sampling mechanism.
Problem (P2) focuses on the role of pooling in eDCNN and studies its theoretical advantages in improving the generalization performance and enabling the feature extraction of eDCNN. With the help of a detailed analysis of the roles of zero-padding and pooling,   theoretical guarantees for the pros and cons of eDCNN, compared with the cDCNN and DFCN, should be
illustrated, which is the main topic of problem (P3). In a nutshell, the above three problems are crucial to understand  the running mechanisms of structured deep nets and probing  into the reason why structured deep nets perform better than  DFCN.

Our purpose in this paper is to provide answers to the aforementioned three problems from the representation theory \cite{bengio2013representation} and statistical learning theory  viewpoints  \cite{gyorfi2002distribution,cucker2007learning}. The main contributions can be concluded as follows:

$\bullet$ Methodology development: We study the role of zero-padding and pooling in eDCNN and find that with suitable pooling strategy, eDCNN  possesses the translation invariance  and  is universal in feature extraction.    These findings show that eDCNN   is better than DFCN in encoding the translation invariance in the network structure without sacrificing its peformance in extracting other features, and is also better than cDCNN  in term of universality in approximation and learning.
Therefore, we actually  provide an alternative network structure for deep learning with more clear running mechanisms and excellent performances in feature extraction and learning.

$\bullet$ Theoretical novelty: We provide solid theoretical verifications on the excellent performance of eDCNN   in   feature extraction and learning. From the feature extraction viewpoint, we prove that zero-padding enables eDCNN to be translation-equivalent and pooling enables it  to be translation-invariant. Furthermore, we   prove that encoding the translation-equivalence (or translation invariance) into the network, eDCNN performs not worse than DFCN in extracting other features in the sense that with similar number of free parameters,    eDCNN  can approximate DFCN within an arbitrary accuracy but not vice-verse. From the learning theory perspective, we prove that eDCNN  is capable of yielding universally consistent learner and encoding the translation-invariance.

$\bullet$ Application guidance: By the aid of  theoretical analysis, we conduct several numerical simulations  on both toy data and real-world applications to show the excellent performance of eDCNN in feature extraction and learning. Our numerical results show that eDCNN  always perform not worse than DFCN and cDCNN. Furthermore, if the data process some  translation-invariance, then eDCNN  is better than other two networks, which provides a guidance on how to use eDCNN.

In summary, we study the feature extraction and learning performances of eDCNN  and provide theoretical answers to problems (P1-P3). For (P1), we prove that zero-padding enables DCNN to reflect the translation-equivalence and improve the performance  of DCNNs in  feature extraction. For (P2), we show that pooling plays a crucial role in reducing the number of parameters of eDCNN without sacrificing its performance in feature extraction. For (P3), we exhibit that if the learning task  includes some translation-equivalence or translation-invariance, eDCNN  is essentially better than other network structures, showing why and when DCNN outperforms DFCN.

The rest of the paper is organized as follows. In the next section, we introduce eDCNN. In Section \ref{Sec.zero-padding}, we study the role of zero-padding in eDCNN. In Section \ref{Sec.pooling},  theoretical analysis is carried out to demonstrate the importance of pooling in eDCNN. In Section \ref{Sec.feature-extraction}, we compare eDCNN with DFCN in feature extraction.
In Section \ref{Sec.Learning}, we verify the  universal consistency of eDCNN and show its translation-invariance in the learning process.
In Section  \ref{sec.Numerical}, numerical experiments concerning both toy simulations and real-world applications are made to illustrate the excellent performance of eDCNN   and verify our theoretical assertions.  In the last section, we draw some conclusions of our results. All proofs of the theoretical assertions  are postponed to Supplementary Material of this paper.

\section{Deep Convolutional Neural Networks}\label{sec.DCNN}
Let  $L\in\mathbb N$ be the depth of a deep net, $d_0=d$ and $d_\ell \in \mathbb{N}$  be the width of the $\ell$-th hidden layer for  $\ell=1,\dots,L$.
For any $\vec{v}\in\mathbb R^{d_{\ell-1}}$, define $\mathcal J_{\ell,W^\ell,\vec{b}^\ell}:\mathbb R^{d_{\ell-1}}\rightarrow\mathbb R^{d_\ell}$ as the affine operator  by
\begin{equation}\label{affine-mapping}
    \mathcal J_{\ell,W^\ell,\vec{b}^\ell}(x):=W^\ell  x+\vec{b}^\ell,
\end{equation}
where $W^\ell$  is a $d_\ell\times d_{\ell-1}$ weight matrix and  $\vec{b}^\ell\in\mathbb R^{d_\ell}$ is a bias vector.
Deep nets with depth $L$ is then defined  by
\begin{equation}\label{deep-net}
     \mathcal N_{d_1,\dots,d_L}(x)
     = \vec{a}^L\cdot  \sigma\circ \mathcal J_{L,W^L,\vec{b}^L} \circ \sigma  \circ \dots \circ \sigma\circ\mathcal J_{1,W^1,\vec{b}^1}(x),
\end{equation}
where $\vec{a}^L \in\mathbb R^{d_L}$, $\sigma(t):=\max\{t,0\}$ is the ReLU function,   $\sigma(x)=(\sigma(x^{(1)}),\dots,\sigma(x^{(d)}))^T$ for $x=(x^{(1)},\dots,x^{(d)})^T$ and $f\circ g(x)=f(g(x))$. Denote by $\mathcal H_{d_1,\dots,d_L}$ the set of deep nets with depth $L$ and width  $d_\ell$ in the $\ell$-th hidden layer.
In this way, deep learning can be regarded as  a  learning scheme that utilizes the feature mapping $\vec{V}_{d_1,\dots,d_L}:\mathbb R^d\rightarrow\mathbb R^{d_L}$ defined by
\begin{equation}\label{data-mappling-deep-learning}
     \vec{V}_{d_1,\dots,d_L}(x):= \sigma\circ \mathcal J_{L,W^L,\vec{b}^L} \circ \sigma  \circ \dots \circ \sigma\circ\mathcal J_{1,W^1,\vec{b}^1}(x)
\end{equation}
to extract data features at first and then uses a simple linear combination of the extracted  features for a  learning  purpose. The quality of learning  is then determined by  the feature mapping  $\vec{V}_{d_1,\dots,d_L}$, which depends on the depth $L$, width $d_\ell$,   bias vectors $\vec{b}^\ell$, and more  importantly, the structure of  weight matrices $W^\ell$, $\ell=1,\dots,L$.
The structure of the weight matrices actually determines the structure of deep nets. For example, full weight matrices correspond to DFCNs \cite{yarotsky2017error}, sparse weight matrices are related to deep  sparsely connected networks (DSCNs) \cite{petersen2018optimal}, Toeplitz-type weight matrices refer to eDCNNs.   Figure~\ref{4networks} presents four structures of deep nets and the associated weight matrices.

\begin{figure*}[htbp]
\centerline{\includegraphics[width=\textwidth]{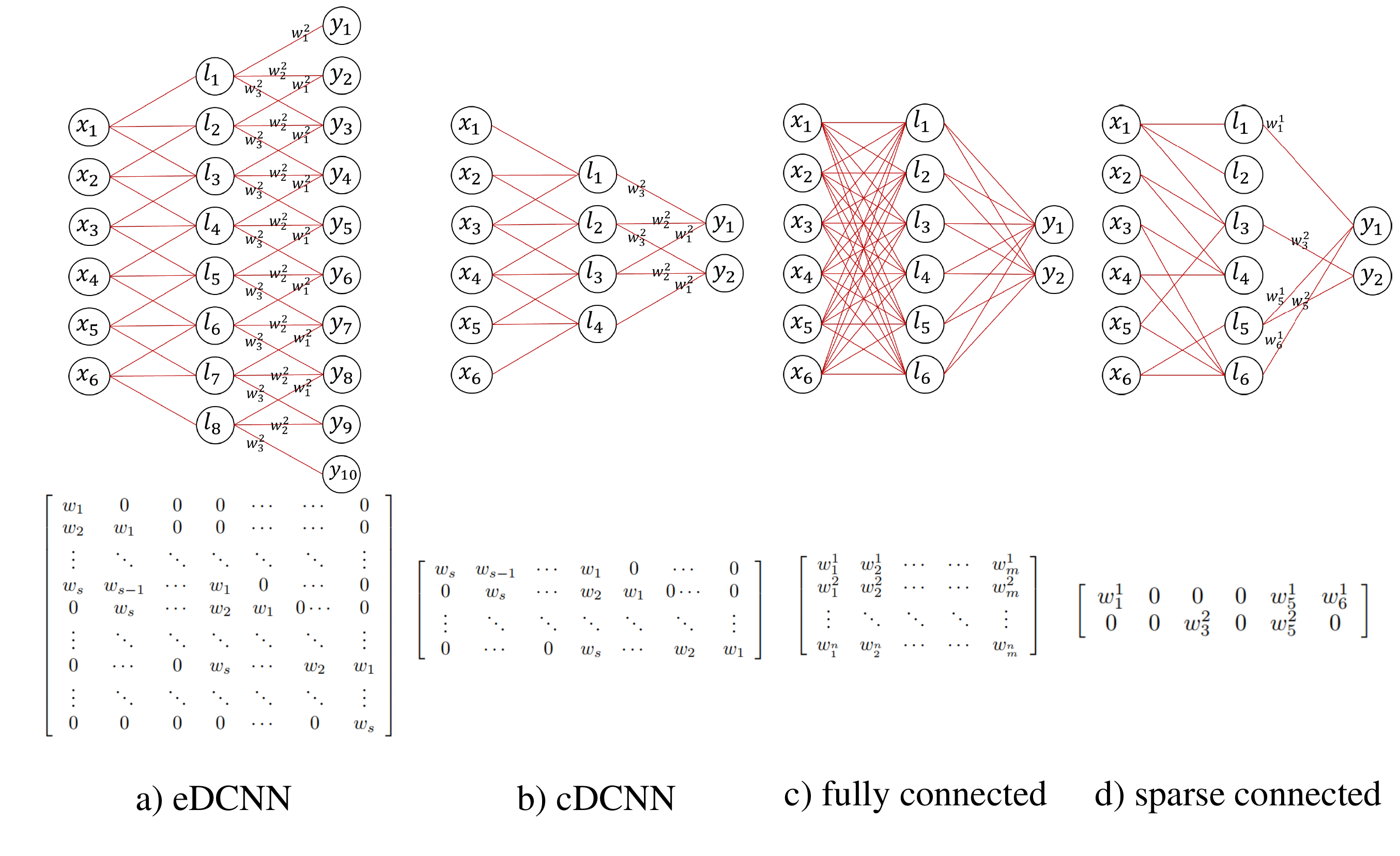}}
\caption{Four structures of eDCNN, cDCNN, DFCN and DSCN.}
\label{4networks}
\end{figure*}

The performances of DFCN and DSCN in feature extraction and learning have been extensively studied in theory\cite{yarotsky2017error,petersen2018optimal,lin2018generalization,schwab2019deep,mhaskar2016deep,schmidt2020nonparametric,kohler2016nonparametric,han2020depth,chui2020realization,lin2017does,safran2017depth,shaham2018provable}. In particular, \cite{shaham2018provable} proved that DFCN succeeds in capturing the manifold structure of the inputs; \cite{safran2017depth} verified that DFCN is capable of realizing numerous features such as the locality and $\ell_1$-radial, and \cite{petersen2018optimal} showed that DSCN benefits in extracting the piece-wise smooth features of the data. Despite these encouraging developments, there is a crucial challenge that the derived provable properties of DFCNs and DSCNs require totally different structures for different learning tasks, not only in different width and depth, but also in the structure of weight matrices. This makes the running mechanism of deep learning still a mystery, since it is quite difficult to design a unified DFCN or DSCN structure suitable for all learning tasks to embody the power of depth.

Noting this, cDCNN  comes into researchers' insights due to its unified structure and popularity in practice \cite{gu2018recent}.  For $s\in\mathbb N$, let $\vec{w}=(w_j)_{j=-\infty}^\infty$ be a filter of   length $s$, i.e.  $w_j\neq0$ only for $0\leq j\leq s$. For any $\vec{v}\in\mathbb R^{d'}$ and $d'\in\mathbb N$, define the one-dimensional and one-channel  convolution without zero padding by
\begin{equation}\label{convolution1}
    (\vec{w}\star\vec{v})_j=\sum_{k=j-s}^{j}w_{j-k}v_{k+s},\qquad j=1,\dots,d'-s.
\end{equation}
For $\vec{v}\in\mathbb R^{d_{\ell-1}}$ and $d_\ell=d_{\ell-1}-s$, define   the contracting convolution operator $\mathcal C^{\star}_{\ell,\vec{w}^\ell,\vec{b}^\ell}:\mathbb R^{d_{\ell-1}}\rightarrow\mathbb R^{d_\ell}$ as
\begin{equation}\label{convolutional-mapping-cDCNN}
    \mathcal C_{\ell,\vec{w}^\ell,\vec{b}^\ell}^\star(\vec{v}):=\vec{w}^\ell\star\vec{v}+\vec{b}^\ell
\end{equation}
for $\vec{w}^\ell$ supported on $\{0,1,\dots,s\}$  and  bias vector  $\vec{b}^\ell\in\mathbb R^{d_\ell}$.  Then cDCNN   is defined by
\begin{equation}\label{cDCNN}
     \mathcal N_{L,s}^{\star}(x)
     := \vec{a}_L\cdot  \sigma\circ \mathcal C_{L,\vec{w}_L,\vec{b}_L}^{\star} \circ \sigma  \circ \dots \circ \sigma\circ\mathcal C^{\star}_{1,\vec{w}_1,\vec{b}_1}(x).
\end{equation}
Due to the contracting nature of cDCNN, we get $d_\ell\leq d$ for all $\ell=1,\dots,L$. This makes cDCNN even not be universal  in approximation, since the smallest width of a deep net required to guarantee the universality is $d+1$ according to the minimal-width theory of deep nets approximation established in \cite{hanin2017approximating}. As a result, though cDCNN is able to extract certain specific features, it is impossible to use the unified cDCNN structure for all learning tasks. A  feasible remedy for this drawback of cDCNN is to add several fully connected layers after the convolutional layers  to enhance the versatility. The problem is, however, that how to set the width and depth of the fully connected layers becomes a complication, making the network structure also unclear. Furthermore, the added  fully connected layers may destroy several important properties such as the translation-equivalence, translation-invariance and calibration invariance of cDCNNs \cite{kayhan2020translation}, making it difficult to theoretically analyze the role of convolution operator \eqref{convolutional-mapping-cDCNN} in the learning process.

Another approach to circumvent the non-universality of cDCNN is to use zero-padding to widen the network, just as \cite{zhou2020universality,zhou2020theory,lin2022universal} did. For any $\vec{v}\in\mathbb R^{d'}$, define
\begin{equation}\label{convolution}
    (\vec{w}*\vec{v})_j=\sum_{k=1}^{d'}w_{j-k}v_k,\qquad j=1,\dots,d'+s.
\end{equation}
Compared with \eqref{convolution1}, zero-padding is imposed in the convolution operation, making the
 convolution  defined by \eqref{convolution} have an expansive   nature,   just as Figure~\ref{fig:zero-padding} purports to show.  For $\vec{v}\in\mathbb R^{d_{\ell-1}}$ and $d_\ell=d_{\ell-1}+s$,
denote the expansive convolution operator $\mathcal C_{\ell,\vec{w}^\ell,\vec{b}^\ell}:\mathbb R^{d_{\ell-1}}\rightarrow\mathbb R^{d_\ell}$ by
\begin{equation}\label{convolutional-mapping}
    \mathcal C_{\ell,\vec{w}^\ell,\vec{b}^\ell}(\vec{v}):=\vec{w}^\ell* \vec{v}+\vec{b}^\ell
\end{equation}
for $\vec{w}^\ell$ supported on $\{0,1,\dots,s\}$  and  bias vector  $\vec{b}^\ell\in\mathbb R^{d_\ell}$. eDCNN  is  then mathematically defined by
\begin{equation}\label{eDCNN}
     \mathcal N_{L,s}(x)
     = \vec{a}_L\cdot \vec{V}^{eDCNN}_{d_1,\dots,d_L},
\end{equation}
where
\begin{equation}\label{def.V}
      \vec{V}^{eDCNN}_{d_1,\dots,d_L}:=\sigma\circ \mathcal C_{L,\vec{w}_L,\vec{b}_L} \circ \sigma  \circ \dots \circ \sigma\circ\mathcal C_{1,\vec{w}_1,\vec{b}_1}(x).
\end{equation}
Denote by $\mathcal H_{L,s}$ the set of all eDCNNs formed as \eqref{eDCNN}.


It is well known that DFCNs do not always obey the commutative law in the sense that there are infinitely many full matrices $A$ and $B$ such that $AB\neq BA$. This
means that the order of the affine operator defined by \eqref{affine-mapping} affects the quality of feature extraction significantly and implies that there must be a strict order for DFCNs to extract different features.
Changing the order of hidden layers thus leads to totally different running mechanisms  of DFCNs. Differently, the convolutional operators defined in \eqref{convolution1} and \eqref{convolution} succeed in breaking through the above bottlenecks of the affine operator by means of admitting the commutative law.

\begin{lemma}\label{lemma:communication}
Let $s\in\mathbb N$. If $\vec{w}^{1},\vec{w}^{2} $ are supported on  $\{0,\dots,s\}$, then
\begin{equation}\label{culmulative-law}
    \vec{w}^{1}\otimes\vec{w}^{2}
   =\vec{w}^{2}\otimes\vec{w}^{1},
\end{equation}
where $\otimes$ denotes either the contracting convolution $\star$ in \eqref{convolution1} or the expansive convolution $*$ in \eqref{convolution}.
\end{lemma}

The commutative law established in Lemma \ref{lemma:communication}   shows that  DCNN  is able to extract different features without considering which feature should be extracted at first, which is totally different from DFCN and presents the outperformance of the convolutional structure  over the classical   inner product structure in DFCN derived from the affine mapping \eqref{affine-mapping}.

We then show the advantage of eDCNN over cDCNN in approximation.   Due to the contracting nature, the depth of cDCNN is always smaller than $d/s$, making the maximal number of  free parameters of cDCNN  not larger than $d$, which is impossible to guarantee the universality \cite{hanin2017approximating}.   The following lemma provided in \cite[Theorem 1]{zhou2020universality} illustrates the universality of eDCNN.

\begin{lemma}\label{Lemma:universality}
 Let $2\leq s\leq d$. There holds
$$
   \lim_{L\rightarrow\infty}
   \inf_{g\in \mathcal H_{L,s}}\|f-g\|_{C(\mathbb I^d)}=0,\qquad\forall f\in C(\mathbb I^d).
$$
\end{lemma}

The universality of eDCNN shows that
 with sufficiently many hidden layers and appropriately tuned weights, eDCNN can extract any other features, which shows its advantage over  cDCNN in approximation.

\section{The Power of Zero-Padding in feature extraction}\label{Sec.zero-padding}
In this section, we   analyze the role of zero-padding to answer problem (P1). Our study starts with the  bottleneck of the contracting convolution structure \eqref{convolution1} in representing the translation-equivalence.
 We say that  a $d$-dimensional vector $\vec{v}_{p,d,j}$ is supported on $\{j,j+1,\dots,j+p-1\}$ with $j+p\leq d+1$, if
\begin{equation}\label{translation}
   \vec{v}_{p,d,j}= (\overbrace{0,\dots,0}^{j-1},v_1,\dots,v_p,\overbrace{0,\dots,0}^{d-p-j+1})^T.
\end{equation}
Let  $A_{j,d}$ be the $d\times d$  matrix whose $(j+i,1+i)$- components   with $i=0,1,\dots,d-j$ are 1 while the others are 0. Then, it is easy to check that $A_{j,d}$ is a
 translation operator (or matrix) satisfying
\begin{equation}\label{role-of-transi}
    A_{j,d} \circ \vec{v}_{p,d,1}=\vec{v}_{p,d,j}.
\end{equation}
We present  definitions of translation-equivalence and translation-invariance \cite{kayhan2020translation} as follows.

\begin{definition}\label{Def:translation}
     Let $\mathcal G_{d',d}:\mathbb R^d\rightarrow\mathbb R^{d'}$ be a linear operator and $A_{j,d}$ be the translation operator satisfying \eqref{role-of-transi}. If
$$
   \mathcal  G_{d',d}\circ A_{j,d}\circ \vec{v}_{p,d,1}=A_{j,d'}\circ\mathcal G_{d',d} \circ \vec{v}_{p,d,1}, \quad \forall j=1,\dots,d-p,
$$
then $\mathcal G_{d',d}$ is said to be  translation-equivalent.
Furthermore, if
$$
   \mathcal  G_{d',d}\circ A_{j,d}\circ \vec{v}_{p,d,1}= \mathcal G_{d',d} \circ\vec{v}_{p,d,1},\quad \forall j=1,\dots,d-p,
$$
then the linear operator  $\mathcal G_{d',d}$ is said to be translation-invariant.
\end{definition}

\begin{figure}[htbp]
\centerline{\includegraphics[width=0.5\textwidth]{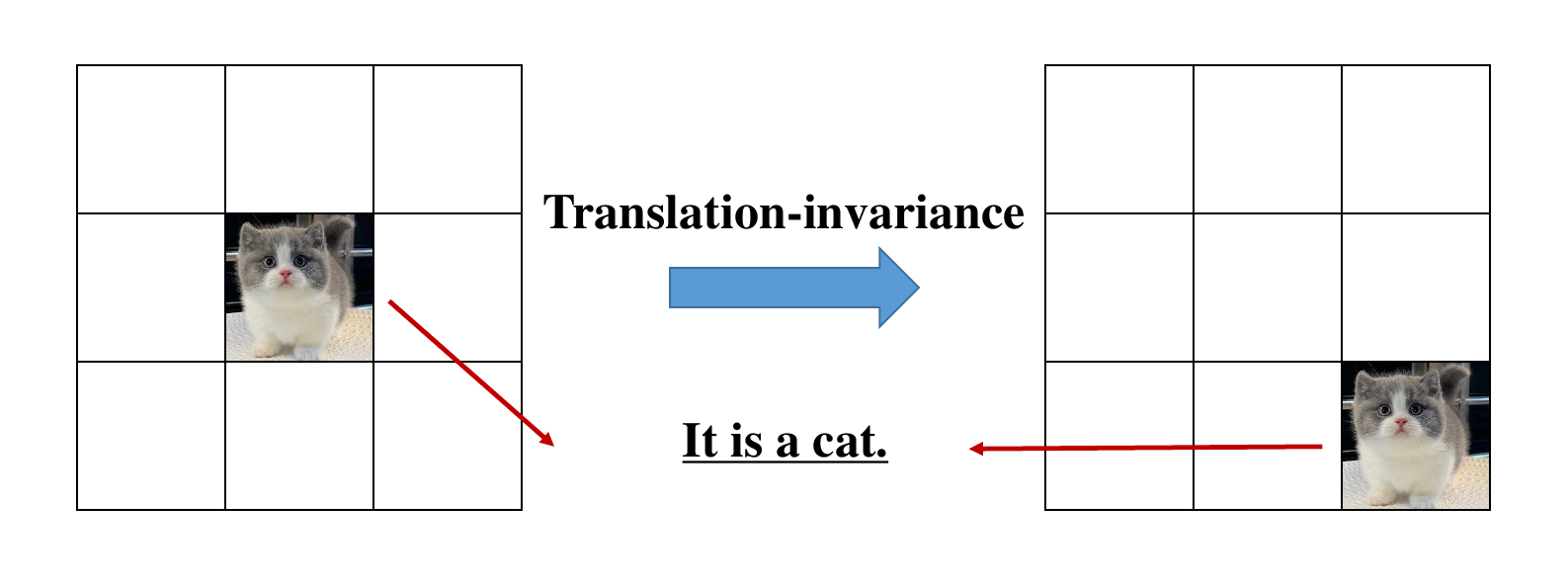}}
\caption{Translation-invariance of image recognition.}
\label{invariance}
\end{figure}

Translation-equivalence and translation-invariance are important features in numerous application regions \cite{kayhan2020translation,bengio2009learning}. Taking image recognition for example, as shown in Figure~\ref{invariance}, a judgement of the cat is easily to  be made independent of  the location of  inputs, showing the translation-invariance of the learning task. Unfortunately, the classical cDCNN is not always capable of    encoding the translation-equivalence
into the network structure, just as the the following lemma exhibits.

\begin{lemma}\label{lemma:bad-translation-equi}
Let $d'\in\mathbb N$, $1\leq p\leq d'$ and $2\leq s\leq d'$. There exist   a $\vec{w}^\ell$ supported on $\{0,1,\dots,s\}$ and some $j\in\{1,\dots,d'-p+1\}$ such
that
$$
    A_{j,d'-s}\circ(\vec{w}^\ell\star\vec{v}_{p,d',1})\neq \vec{w}^\ell\star(A_{j,d'}\circ \vec{v}_{p,d',1}).
$$
\end{lemma}

It is easy to check that there is a $\vec{w}^\ell$  supported on $\{0,1,\dots,s\}$ such that there are $s$ non-zero items in $\vec{w}^\ell\star\vec{v}_{p,d',1}$ and consequently $s$ non-zero  items in $A_{j,d'-s} \circ (\vec{w}^\ell\star\vec{v}_{p,d',1})$,
but  for $s\leq j\leq d'-p-s$ there are $2s-1$ non-zero items in $\vec{w}^\ell\star(A_{j,d'}\circ \vec{v}_{p,d',1})$, which proves the above lemma directly. Lemma \ref{lemma:bad-translation-equi} does not show that cDCNN  always cannot encode the translation-invariance, but demonstrates that cDCNN limits to do this when the support of the target   lies on the edges.
Our   result in this section is to show that zero-padding defined by \eqref{convolution} is capable of breaking through the above drawback of cDCNN.

\begin{proposition}\label{Proposition:translation-equiv}
 Let $L\in\mathbb N$, $1\leq p\leq d$, $2\leq s\leq d$, $d_0=d$, $d_\ell=d+\ell s$ and  $\vec{w}^\ell$  be supported on $\{0,1,\dots,s\}$, $\ell=1,\dots,L$. For any $1\leq j\leq d-p+1$, there holds
\begin{eqnarray}\label{trans-equiv}
  &&\vec{w}^L*\dots *\vec{w}^1*(A_{j,d}\circ\vec{v}_{p,d,1})\nonumber\\
  &=&
  \vec{w}^L*\dots * (A_{j,d_1}\circ(\vec{w}^1*\vec{v}_{p,d,1})) \nonumber\\
  &=&\dots=
  A_{j,d_L}\circ(\vec{w}^L*\dots *\vec{w}^1*\vec{v}_{p,d,1})
 .
\end{eqnarray}
\end{proposition}

The proof of Proposition \ref{Proposition:translation-equiv} will be given in Appendix.  Proposition \ref{Proposition:translation-equiv} shows that if a vector $v_{p,d,j}$ is translated to $v_{p,d,j'}$ for $j'\neq j$, then the output of the convoluted vector $\vec{w}^L*\dots *\vec{w}^1*\vec{v}_{p,d,j}$ is also translated with step $j'-j$. This illustrates that without tuning weights, the expansive convolution structure succeeds in encoding the translation-equivalence  feature of the data, showing its outperformance over   the inner product structure in DFCNs and the contracting convolutional structure in cDCNNs.

\section{The Importance of Pooling in eDCNN}\label{Sec.pooling}



In this section, we borrow the idea of  the location-based pooling (down-sampling) scheme from \cite{zhou2020theory}  to equip eDCNN to reduce the size of extracted features and enhance its performance of encoding the translation-invariance, simultaneously.  For any $  d'\in\mathbb N$, the location-based pooling  scheme  $\mathcal S_{d',u,j}:\mathbb R^{d'}\rightarrow \mathbb R^{[d'/u]}$ for a vector $\vec{v}\in\mathbb R^{d'}$ with   scaling parameter $u$ and location parameter $0\leq j\leq d'$  is defined by
\begin{equation}\label{def.down-sampling}
    \mathcal S_{d',u,j}(\vec{v})=(v_{ku+j})_{k=1}^{[d'/u]},
\end{equation}
where $[a]$ denotes the integer part of the real number $a$ and $v_{ku+j}=0$. For any $d\in\mathbb N$, if we set $v_{ku+j}=0$ for $ku+j>d$,  then $\mathcal S_{d',u,j}$  can be regarded as a operator from $\mathbb R^{d'}$ to $\mathbb R^d$.
 As shown in Figure~\ref{pooling} , $\mathcal S_{d',u,j}$ devotes to selecting  $[d'/u]$ neurons from $d'$ features and the selection rule is only based on the location of the neurons,
which is totally different from the classical max-pooling that selects   neurons with the largest value or the average-pooling  which synthesizes the averaged value of several neurons.

\begin{figure}[htbp]
\centerline{\includegraphics[width=0.5\textwidth]{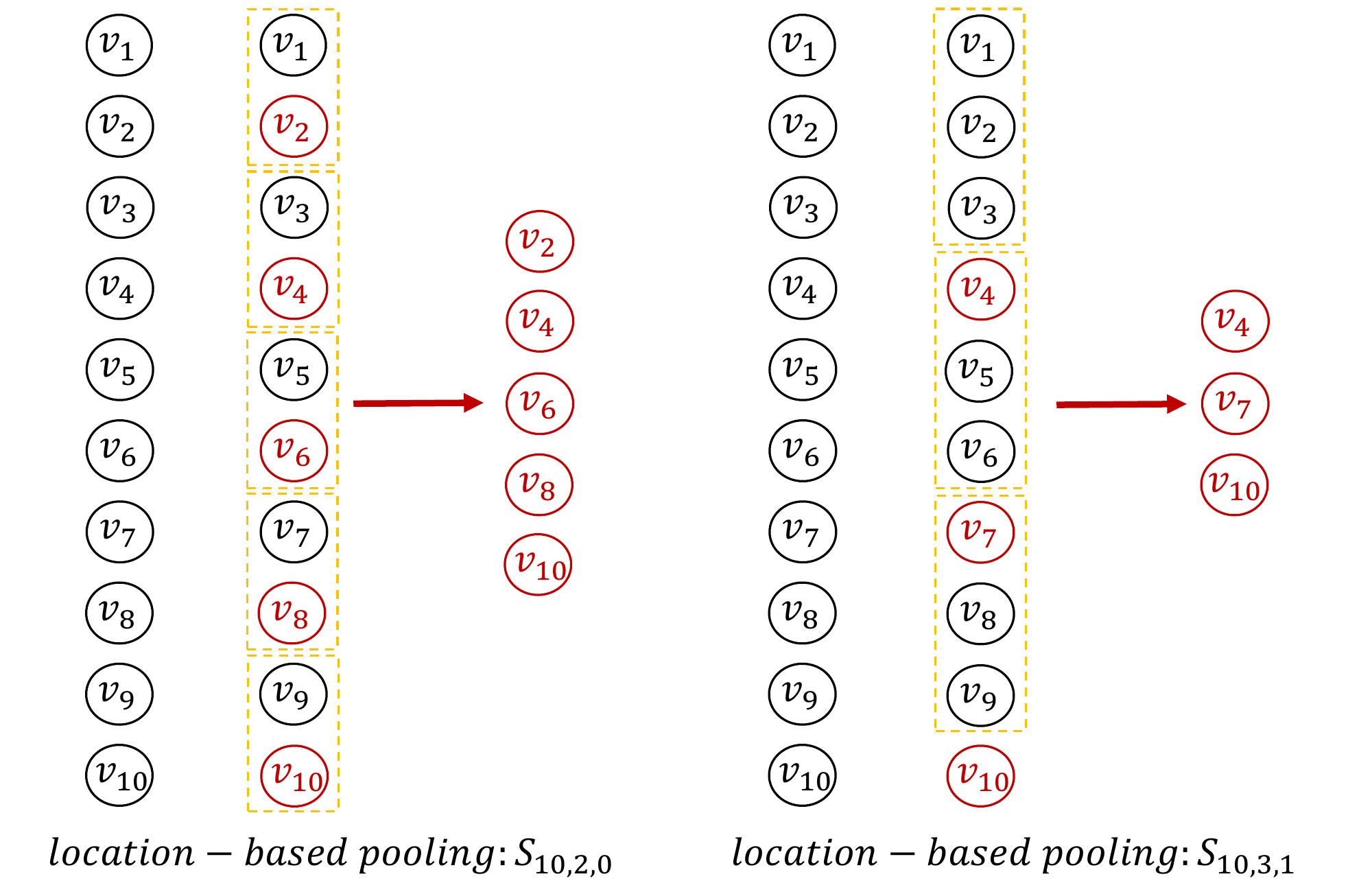}}
\caption{Location-based pooling.}
\label{pooling}
\end{figure}

With the proposed location-based pooling scheme, we  show that eDCNN succeeds in encoding the translation-equivalence (or translation-invariance) without sacrificing the  performance of DCFN in extracting other features. Given $d',\tilde{d},\ell\in\mathbb N$ satisfying $\tilde{d}=d'+\ell s$, $\vec{b}\in\mathbb R^{d'}$ and $\vec{w}^{1},\dots,\vec{w}^{L}$ supported on $\{0,\dots s\}$ with $L=\lceil\frac{\tilde{d}d'}{s-1}\rceil$,  for any $x\in\mathbb R^{d'}$, define  a multi-layer convolutional operator with pooling by
\begin{equation}\label{multi-layer-convolutional}
   \mathcal B_{L,d',k}(x):= \mathcal S_{\tilde{d},d',k}( \vec{w}^{L}*\dots*\vec{w}^{1}*x+\vec{b}).
\end{equation}
The following proposition is the main result of this section.

\begin{proposition}\label{Proposition:C-Struc-fea-deep}
Let $J,d\in\mathbb N$, $2\leq s\leq d$, $d_0=d$, $d_1,\dots,d_J\in\mathbb N$,  $L_j=\lceil\frac{d_{j-1}d_j}{s-1}\rceil$, and $x$ be supported on $\{k,k+1,\dots,k+p-1\}$ for some $p,k\in\mathbb N$. If $\vec{V}_{d_1,\dots,d_L}:\mathbb R^d\rightarrow\mathbb R^{d_L}$ defined by \eqref{data-mappling-deep-learning} is an arbitrary feature generated by DFCN, then
there exist $\sum_{j=1}^JL_j$   filter vectors $\{\vec{w}^{j,\ell}\}_{\ell=1}^{L_j}$ supported on $\{0,1,\dots,s\}$ and $J$ bias vectors $\vec{b}^j\in\mathbb R^{d_j}$, $j=1,\dots,J$ such that for any $1\leq i\leq d-p+1-k$ there holds
\begin{eqnarray}\label{cons-stru-deep}
     &\vec{V}_{d_1,\dots,d_J}(x)
       =
       \sigma\circ \mathcal B_{L_J,d_{J-1},0}
       \circ  \dots \circ   \sigma\circ\mathcal B_{L_1,d_0,0}(x) \\
       & =
   \sigma \circ \mathcal B_{L_J,d_{J-1},0}
       \circ  \dots \circ \sigma\circ \mathcal B_{L_2,d_1,0}\circ \sigma\circ \mathcal B_{L_1,d_1,i}(A_{i,d}x). \nonumber
\end{eqnarray}
\end{proposition}

 Due to the commutative law established in Lemma \ref{lemma:communication}, it is easy to check that  the representation presented in \eqref{cons-stru-deep} is not unique. Noting that there are
 $$
    L_j(s+1)=\lceil\frac{d_{j-1}d_j}{s-1}\rceil(s+1)\leq3d_jd_{j-1},
 $$
free parameters involved in the multi-layer convolutional structure $\vec{w}^{j,L_j}*\dots*\vec{w}^{j,1}$,
which is comparable with those in the classical inner product structure, i.e., $d_jd_{j-1}$, the location-based pooling scheme  is to enable the same size of the affine-mapping and multi-layer convolutional mapping without sacrificing its performance in feature extraction.
 Therefore,
Proposition \ref{Proposition:C-Struc-fea-deep} yields that acting as a feature extractor, the convolutional structure together with suitable pooling schemes outperforms the classical inner product structure in DFCN in  the sense that the former succeeds in capturing the translation-invariance structure without losing the performance of DFCN in extracting other features. A direct corollary of the above proposition is the following translation-invariance of eDCNN with pooling.

\begin{corollary}\label{Corollary:trans-invar}
 Let $L\in\mathbb N$, $2\leq s\leq d$, $1\leq p\leq d$ and  $\vec{w}^\ell$ with $\ell=1,\dots,L$ be supported on $\{0,1,\dots,s\}$. Then  for any $1\leq j,j'\leq d-p+1$ with $j'\neq j$, there holds
\begin{equation}\label{trans-invar}
\mathcal S_{d_L,d,j}(\vec{w}^L*\dots *\vec{w}^1*\vec{v}_{p,d,j})
  =
  \mathcal S_{d_L,d,j'}(\vec{w}^L*\dots *\vec{w}^1*\vec{v}_{p,d,j'}).
\end{equation}
\end{corollary}

It follows from Corollary \ref{Corollary:trans-invar} that a suitable pooling scheme not only reduces the free parameters in a network but also endows the network with some important property.  Therefore, if the convolutional structure is employed in the network, designing an appropriate pooling scheme is crucial to improve its  performance in feature extraction and learning. Another direct corollary of Proposition \ref{Proposition:C-Struc-fea-deep} is the following comparison between eDCNN with pooling and shallow nets.

\begin{corollary}\label{Corollary:C-struc-In-F}
Let $2\leq s\leq d$, $n\in\mathbb N$, $x$ be supported on $\{j,j+1,\dots,j+p-1\}$ for some $p,j\in\mathbb N$,  $\sigma(Wx+\vec{b})$ be a feature generated by a shallow net  with    $n\times d$ weight matrix $W$ and   $n$ dimensional bias vector $\vec{b}$. Then, there exist  $L^*=\lceil\frac{nd}{s-1}\rceil$ filter vectors $\{\vec{w}^\ell\}_{\ell=1}^{L^*}$ supported on $\{0,1,\dots,s\}$and a bias vector $\vec{b}'\in\mathbb R^n$ such that  for any $1\leq k\leq d-p+1-j$, there holds
\begin{eqnarray*}
   \sigma(Wx+\vec{b})
    =
   \sigma \circ\mathcal B_{L^*,d,0}(x)
    =
   \sigma \circ\mathcal B_{L^*,d,k}(A_{k,d}x)),
\end{eqnarray*}
where $A_{k,d}$ is the translation matrix satisfying \eqref{role-of-transi}.
\end{corollary}

Corollary \ref{Corollary:C-struc-In-F} shows that  convolutional structures and appropriate pooling scheme can provide the translation-invariance
without sacrificing the feature extraction capability of the classical shallow nets. It implies that no matter where $x$ is supported, the multi convolutional structure together with suitable location-based pooling scheme yields the same output of high quality.

\section{Performance of eDCNNs in Feature Extraction  }\label{Sec.feature-extraction}

As discussed in the above two sections, zero-padding and location-based pooling play important roles in improving the performance  of the convolutional structure in feature extraction. However, the network provided in Proposition \ref{Proposition:C-Struc-fea-deep} uses both bias vectors and ReLU activation functions and is thus different from  the eDCNN defined by \eqref{eDCNN}. In this section, we show that adding suitable bias vectors enables to produce a unified  eDCNN structure formed as \eqref{eDCNN}. We then present a comparison between eDCNN and DFCN in  feature extraction.

Since adopting  different bias vectors in different neurons in the same layer    prohibits  the translation-equivalence of eDCNN,   we study a subset of $\mathcal H_{L,s}$ whose elements of the bias vector in the same layer  are the same. For any $d'_0=d'\in\mathbb N$ and $d_\ell'=d'+\ell s$, define  the   restricted convolution operator by
\begin{equation}\label{convolutional-mapping-res}
    \mathcal C^R_{\ell,\vec{w}^\ell,b^\ell}(x):=\vec{w}^\ell* x+b^\ell {\bf 1}_{d'_\ell}
\end{equation}
for $\vec{w}^\ell$ supported on $\{0,1,\dots,s\}$,    $b^\ell\in\mathbb R$ and $\mathbf 1_{d'_\ell}=(1,\dots,1)^T\in\mathbb R^{d'_\ell}$.
Define further the restricted eDCNN feature-extractor from $\mathbb R^{d'}$ to $\mathbb R^{d'_\ell}$ by
\begin{equation}\label{r-eDCNN-feature-extractor}
    \vec{V}_{\ell}^{s,R}(x):=\sigma\circ \mathcal C^R_{\ell,\vec{w}^{\ell}, {b}^{\ell}} \circ \sigma  \circ \dots \circ \sigma\circ\mathcal C^R_{1,\vec{w}^1, {b}^1}(x).
\end{equation}
  Denote by
\begin{eqnarray}\label{r-eDCNN-pooling}
   && \vec{V}_{\ell}^{s,R,d',j}(x)
     :=
    \mathcal S_{d'_{\ell},d'_0,j}\circ \sigma \circ \mathcal C_{\ell,\vec{w}^{\ell}, \vec{b}^{\ell}} \circ \sigma\circ \mathcal C^R_{\ell-1,\vec{w}^{\ell-1}, {b}^{\ell-1}} \nonumber\\
    &&\circ \sigma\circ \mathcal C^R_{\ell-2,\vec{w}^{\ell-2}, {b}^{\ell-2}}
    \circ \dots \circ \sigma\circ\mathcal C^R_{1,\vec{w}^1, {b}^1}(x)
\end{eqnarray}
the restricted eDCNN feature-extractor with depth $\ell$, filter length $s$, location-based pooling at the $d_{\ell}$ layer  with pooling parameter $d'j\in\mathbb N$.  It should be highlighted that when pooling happens, that is, in the $\ell$-th layer,  we use the   convolutional operator \eqref{convolutional-mapping} rather than the restricted one \eqref{convolutional-mapping-res}.
In the following theorem, we show that eDCNN with  pooling performs not worse than shallow nets with  similar number of parameters in feature extraction but succeeds in encoding the translation-invariance.


\begin{theorem}\label{Theorem:app-dcnn-res}
Let $2\leq s\leq d$, $n\in\mathbb N$, $L^*=\lceil\frac{nd}{s-1}\rceil$, $d_0=d$, and $d_\ell=d_{\ell-1}+s$. For any $W\in\mathbb R^{n\times d}$ and $\vec{\theta}\in\mathbb R^n$, there exist $b^\ell\in\mathbb R$ for $1\leq \ell\leq L^*-1$, $\vec{b}^{L^*}\in\mathbb R^{d_{L^*}}$ and $L^*$ filter vectors $\vec{w}^\ell$ supported on $\{0,1,\dots,s\}$ such that
\begin{eqnarray}\label{eDCNN-for-shallow}
 \sigma(Wx+\vec{\theta}) = \vec{V}_{L^*}^{s,R,d,0}(x)
\end{eqnarray}
If in addition $x$ is supported on $\{j,j+1,\dots,j+p-1\}$ for some $p,j\in\mathbb N$, then for any $1\leq k\leq d-p+1$, there holds
\begin{eqnarray}\label{trans-inv-eDCNN-1}
   \vec{V}_{L^*}^{s,R,d,0}(x)=\vec{V}_{L^*}^{s,R,d,k}(A_{k,d}x).
\end{eqnarray}
\end{theorem}

Theorem \ref{Theorem:app-dcnn-res} presents three advantages of eDCNN over shallow nets. At first, as far as the feature extraction is concerned, eDCNN performs not worse than shallow nets in the sense that it can exactly represent any features extracted by shallow nets with comparable free parameters. Then, it follows from
\eqref{trans-inv-eDCNN-1} that with suitable pooling mechanism, eDCNN succeeds in encoding the translation-invariance into the structure without tuning weights, which is totally different from other network structures. Finally, it can be derived from \eqref{eDCNN-for-shallow} and \eqref{trans-inv-eDCNN-1}  that even with the structure constraints,  the approximation capability of eDCNN is at least not worse than shallow nets.  In particular, denote by
\begin{equation}\label{eDCNN-space-shallow}
 \mathcal H_{\ell}^{s,R,d,j}:=
 \left\{h(x)=\vec{a}\cdot \vec{V}_{\ell,\ell-1}^{s,R,d,0}(x):\vec{a},\vec{b}^\ell\in\mathbb R^n,b^{k}\in\mathbb R\right\}
\end{equation}
the set of restricted eDCNN with pooling defined by
\eqref{r-eDCNN-pooling} with
$\{\vec{w}^k\}_{k=1}^\ell$ being supported on $\{0,1,\dots,s\}$.
We can derive the following corollary from Theorem \ref{Theorem:app-dcnn-res} directly.

\begin{corollary}\label{Corollary:shallow-edCNN-app}
Let $2\leq s\leq d$, $n\in\mathbb N$, $L^*=\lceil\frac{nd}{s-1}\rceil$, $d_0=d$ and $d_\ell=d_{\ell-1}+s$. Then, for any $f\in C(\mathbb I^d)$, we have
\begin{equation}\label{shallow-edcnn-app}
    \mbox{dist}(f,\mathcal H_{\ell}^{s,R,d,j})\leq
    \mbox{dist}(f,\mathcal H_{n}),
\end{equation}
where $\mbox{dist}(f,\mathcal H):=\inf_{g\in\mathcal H}\|f-g\|_{C(\mathbb I^d)}$ denotes the distance between the function $f$ and the set  $\mathcal H$ in $C(\mathbb I^d)$ and $\mathcal H_n$ is the set of shallow nets with $n$ neurons in the hidden layer.
\end{corollary}

Theorem \ref{Theorem:app-dcnn-res} and Corollary \ref{Corollary:shallow-edCNN-app} illustrate the outperformance of eDCNN over shallow nets.  In the following, we aim to compare eDCNN with DFCN in feature extraction. The following theorem is our second main result.

\begin{theorem}\label{Theorem:comp-app-dfcn}
Let $2\leq s\leq d$, $L\in\mathbb N$, $d_1^*,\dots,d_L^*\in\mathbb N$ and $d_0^*=d$. If $L_\ell=\lceil\frac{d^*_\ell d^*_{\ell-1}}{s-1}\rceil$ for $\ell=1,2,\dots,L$, then  for any  $\vec{V}_{d_1^*,\dots,d_L^*}(x)$ defined by \eqref{data-mappling-deep-learning}, there exist filter sequences $\{\vec{w}_{\ell,j}\}_{1,1}^{L,d_{L_{\ell}}}$ with $d_{L_{\ell}}=d+L_{\ell}s$ supported on $\{0,1,\dots,s\},$  $b_{\ell,j}\in \mathbb R$ and  $\vec{b}^{L_\ell }\in\mathbb R^{d^*_{L_{\ell}}}$ such that
\begin{equation}\label{edcnn-dfcn-app}
      \vec{V}_{d_1^*,\dots,d_L^*}(x)
  =
  \vec{V}_{L_L}^{s,R,d^*_{L-1},0}\circ\vec{V}_{L_{L-1}}^{s,R,d^*_{L-2},0}\circ\dots
  \circ \vec{V}_{L_1}^{s,R,d^*_{0},0}(x).
\end{equation}
If in addition $x$ is supported on $\{j,j+1,\dots,j+p-1\}$ for some $p,j\in\mathbb N$, then for any $1\leq k\leq d-p+1$, there holds
\begin{eqnarray}\label{trans-inv-edcnn-for-dfcn}
    &&\vec{V}_{L_L}^{s,R,d^*_{L-1},0}\circ\vec{V}_{L_{L-1}}^{s,R,d^*_{L-2},0}\circ\dots
  \circ \vec{V}_{L_1}^{s,R,d^*_{0},0}(x)\\
  &=&
  \vec{V}_{L_L}^{s,R,d^*_{L-1},0}\circ\vec{V}_{L_{L-1}}^{s,R,d^*_{L-2},0}\circ\dots
  \circ \vec{V}_{L_1}^{s,R,d^*_{0},k}(A_{k,d}x).\nonumber
\end{eqnarray}
\end{theorem}

 It should be mentioned that the only difference between eDCNN for shallow nets and eDCNN for DFCN is the number of pooling. In fact, to represent an $L$-layer DFCN, it requires $L$ location-based pooling in eDCNN.
Theorem \ref{Theorem:comp-app-dfcn} shows that eDCNN with suitable pooling scheme can represent any DFCN with comparable free parameters, even though eDCNN is capable of encoding the translation-invariance in the network structure. This demonstrates the outperformance of eDCNN over DFCN. Denote by $\mathcal H^{s,R,d^*_{L-1},\dots,d^*_0,0}_{L_L,\dots,L_1}$
the set of all eDCNNs formed as
$$
  \vec{a}\cdot \vec{V}_{L_L}^{s,R,d^*_{L-1},0}\circ\vec{V}_{L_{L-1}}^{s,R,d^*_{L-2},0}\circ\dots
  \circ \vec{V}_{L_1}^{s,R,d^*_{0},0}(x).
$$
We obtain from Theorem \ref{Theorem:comp-app-dfcn} the following corollary directly.

\begin{corollary}\label{Corollary:DFCN-edCNN-app}
Let $2\leq s\leq d$, $L\in\mathbb N$, $d_1^*,\dots,d_L^*\in\mathbb N$ and $d_0^*=d$. If $L_\ell=\lceil\frac{d^*_\ell d^*_{\ell-1}}{s-1}\rceil$ for $\ell=1,2,\dots,L^*$, then  for any    $f\in C(\mathbb I^d)$, we have
\begin{equation}\label{DFCN-edcnn-app}
    \mbox{dist}(f,\mathcal H^{s,R,d^*_{L-1},\dots,d^*_0,0}_{L_L,\dots,L_1})\leq
    \mbox{dist}(f,\mathcal H_{d^*_1,\dots,d^*_L}).
\end{equation}
\end{corollary}

With the help of the location-based pooling scheme  and convolutional structures, it is easy to check that the number of free parameters involved in $\mathcal H^{s,R,d^*_{L_1},\dots,d^*_0,0}_{L_L,\dots,L_1}$  and $\mathcal H_{d^*_1,\dots,d^*_L}$ are of the same order, which implies that the approximation capability of the eDCNN is at least not worse than that of DFCN. The problem is, however, that it is difficult to set the sizes of pooling  since $\{d_j^*\}_{j=1}^L$ for a specific learning task is usually unknown, making them  be   crucial parameters in eDCNN.   The following corollary devotes to the universal approximation property of eDCNN with specified $\{d_j^*\}_{j=1}^L$.


\begin{corollary}\label{Corollary:mult-app-structure}
Let $2\leq s\leq d$ and  $\ell_1=\lceil\frac{d(d+1)}{s-1}\rceil$ and $\ell_j=\lceil\frac{(d+1)^2}{s-1}\rceil$ for $j\geq 2$. For an arbitrary $\varepsilon>0$ and any $f\in C(\mathbb I^d)$, there exist  an $L\in\mathbb N$  such that
$$
    \mbox{dist}(f,\mathcal H^{s,R,d+1,\dots,d+1,0}_{\ell_L,\dots,\ell_1})<\varepsilon.
$$
\end{corollary}

Corollary \ref{Corollary:mult-app-structure} can be derived directly from Corollary  \ref{Corollary:DFCN-edCNN-app} and \cite{hanin2017approximating} where the universality of DFCN with fixed width $d+1$ was verified. Corollary \ref{Corollary:mult-app-structure} shows that, with the unified   structure and specific
pooling scheme in the sense that the pooling is made as soon as the width of the network achieves $\lceil\frac{(d+1)^2}{s-1}\rceil$, just as Figure \ref{loc pooling} purports to show,  eDCNN with location-based pooling is universal in approximation.

\begin{figure}[htbp]
\centerline{\includegraphics[width=0.48\textwidth]{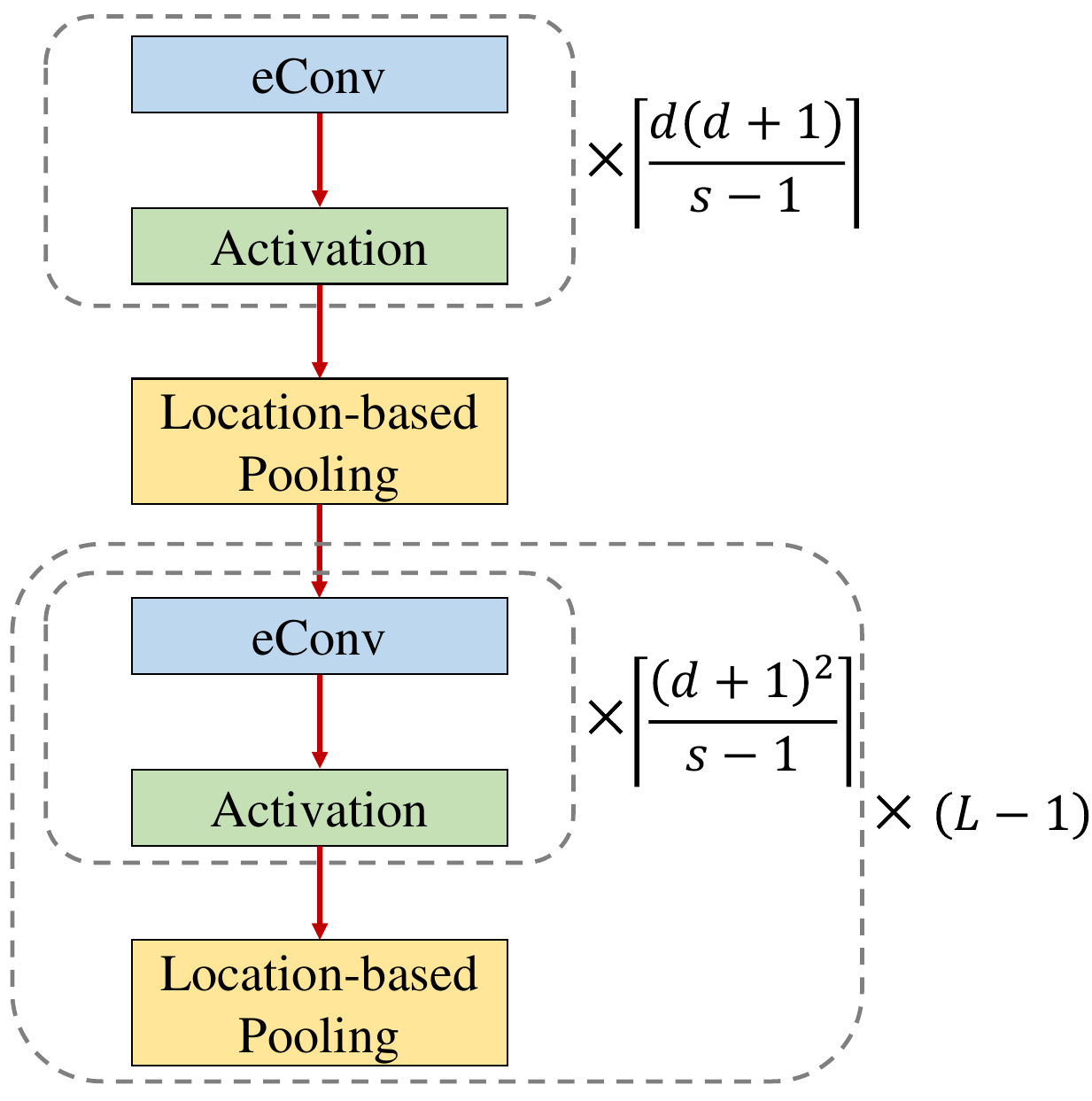}}
\caption{The unified structure and specific pooling scheme.}
\label{loc pooling}
\end{figure}


\section{ Performance of eDCNN in Learning}\label{Sec.Learning}
In this section, we study the learning performance of eDCNN in terms of verifying its strongly universal consistency \cite{gyorfi2002distribution,lin2022universal}. We conduct our analysis in the standard least-square regression framework \cite{gyorfi2002distribution,cucker2007learning} for the sake of brevity, although our results also hold for other loss functions.

In least-square regression \cite{gyorfi2002distribution,cucker2007learning},   samples in the data set $D:=\{z_i\}_{i=1}^m:=\{(x_i,y_i)\}_{i=1}^m$ are assumed to be drawn independently and identically (i.i.d.) from an unknown but definite  probability distribution $\rho$ on $Z:=\mathcal X\times\mathcal Y$, where $\mathcal X=\mathbb I^d$, and $\mathcal Y\subseteq\mathbb R$.   The aim is to learn a function $f_D$ based on $D$ to  minimize the generalization error
$
        \mathcal E(f):=\int_{\mathcal Z}(f(x)-y)^2d\rho,
$
which is theoretically minimized by the well known regression function \cite{gyorfi2002distribution} $f_\rho(x):=\int_{\mathcal Y}yd\rho(y|x)$. Since $\rho$ is unknown, its conditional mean $f_\rho$ is also unknown.
Our aim is then to find an estimator $f_D$ to minimize
\begin{equation}\label{equality}
      \mathcal E(f)-\mathcal E(f_\rho)=\|f-f_\rho\|_{L_{\rho_X}^2}^2,
\end{equation}
where $\rho_X$ is the marginal distribution of $\rho$ on ${\mathcal X}$.

One of the most important properties that a learner should have  is
that, as the sample size $m$ grows, the deduced estimator
converges to the real relation between the input and output. This property, featured as
the strongly universal consistency \cite{gyorfi2002distribution},  can be defined as follows.

\begin{definition}\label{DEFINITION WEAKLY UNIVEAL}
 A
sequence of regression estimators $\{f_m\}_{m=1}^\infty$ is called
strongly universally consistent, if
$$
       \lim_{m\rightarrow\infty}\mathcal E(f_m)-\mathcal E(f_\rho)=0
$$
holds with probability one for all Borel probability distributions $\rho$ satisfying
 $\int_{\mathcal Y}y^2d\rho(y|x) <\infty$.
\end{definition}

In our previous work \cite{lin2022universal} we show  that running empirical risk minimization (ERM) on eDCNN without pooling yields strongly universally consistent learners. In this section, we  show that imposing pooling scheme on eDCNN can maintain the strongly universal  consistency. For this purpose,
 we build up the learner via ERM:
\begin{equation}\label{ERM}
     f^{s,R,d^*_{L-1},\dots,d^*_0,0}_{D,L_L,\dots,L_1}:={\arg\min}_{f\in\mathcal H^{s,R,d^*_{L-1},\dots,d^*_0,0}_{L_L,\dots,L_1}}\mathcal E_D(f),
\end{equation}
where
$
     \mathcal E_D(f)=\frac1{|D|}\sum_{i=1}^{|D|}(f(x_i)-y_i)^2
$
denotes the empirical risk of $f$.  As shown in Definition \ref{DEFINITION WEAKLY UNIVEAL},
expect for $\int_{\mathcal Y}y^2d\rho(y|x) <\infty$, there is not  any other restriction  on $y$, making the analysis be quite  difficult. A preferable way is to consider a truncation operator defined by $\pi_M t = \min\{M,|t|\}\cdot\mathrm{sgn}(t)$  on $y$, $y_{M_D}:=\pi_{M_D} y$  with $M_D\rightarrow\infty$.
Therefore, our final estimator is $\{\pi_{M_D}f^{s,R,d^*_{L-1},\dots,d^*_0,0}_{D,L_L,\dots,L_1}\}_{|D|=1}^\infty$.
It should be mentioned that such a truncation operator is widely adopted in demonstrating the universal consistency for numerous learning schemes \cite{gyorfi2002distribution,lin2022universal}, including local average regression, linear least squares,  shallow neural networks learning and eDCNN without pooling.
The following theorem  then presents   sufficient conditions on designing the pooling scheme and selecting the depth of eDCNN to guarantee the strongly universal  consistency.

\begin{theorem}\label{Theorem:universal consistency}
Let $\theta\in(0,1/2)$ be an arbitrary  real number, $2\leq s\leq d$,  and $L:=L_D\in\mathbb N$, $d_1^*,\dots,d_L^*\in\mathbb N$ possibly depending on $|D|$.
 If $d_1^*,\dots,d_L^*\geq d+1$, $\sum_{j=1}^Ld_{j-1}^*d_j^*\rightarrow\infty$, $M =M_{D}\rightarrow\infty$, $M^2_{D}|D|^{-\theta}\rightarrow0$ and
\begin{equation}\label{condition of  universal}
           \frac{M_D^4\left(\sum_{j=1}^Ld_{j-1}^*d_j^*\right)^2 \log (L_Dd^*_{\max}) \log(M_D|D|)}{|D|^{1-2\theta}}\rightarrow0,
\end{equation}
then $\pi_{M_{D}}f^{s,R,d^*_{L-1},\dots,d^*_0,0}_{D,L_L,\dots,L_1}$ is strongly universally consistent.
If in addition $x$ is supported on $\{j,j+1,\dots,j+p-1\}$ for some $p,j\in\mathbb N$, then for any $1\leq k\leq d-p+1$, there holds
\begin{eqnarray}\label{trans-inv-universal}
    \pi_{M_{D}}f^{s,R,d^*_{L-1},\dots,d^*_0,0}_{D,L_L,\dots,L_1}(x)=
    \pi_{M_{D}}f^{s,R,d^*_{L-1},\dots,d^*_0,k}_{D,L_L,\dots,L_1}(A_{k,d}x).
\end{eqnarray}
\end{theorem}

Theorem \ref{Theorem:universal consistency} presents  sufficient conditions on the structure and pooling scheme for verifying the strongly universal consistency of eDCNN. Different from \cite{lin2022universal}, Theorem \ref{Theorem:universal consistency} admits different pooling schemes to guarantee the universal consistency. Theorem \ref{Theorem:universal consistency} shows that with arbitrary $L\rightarrow\infty$ location-based pooling operators, eDCNN can reduce the width of network from infinite to finite.  The role of the location-based pooling is not only to trim the network and consequently reduces the number of parameters, but also to enhance the translation-invariance of the network structure as shown in \eqref{trans-inv-universal}. If we set $L=1$, then Theorem \ref{Theorem:universal consistency} is similar as   \cite[Theorem 1]{lin2022universal}.  If we set $d_1^*=\dots=d_L^*=d+1$, then the condition \eqref{condition of  universal} reduces to
\begin{equation}\label{condition of  universal-1}
           \frac{M_D^4L_D^2 \log (L_D) \log(M_D|D|)}{|D|^{1-2\theta}}\rightarrow0,
\end{equation}
which shows that eDCNN  with the structure in Corollary \ref{Corollary:mult-app-structure} (or Figure \ref{loc pooling}) is universally consistent.

\section{Numerical Verifications}\label{sec.Numerical}

In this section, we conduct both toy simulations and real data experiments to verify the excellent performance of eDCNN, compared with cDCNN and  DFCN.  In all these   simulations, we train the networks with an Adam optimizer for 2000 epochs. The piece-wise learning rates are [0.003, 0.001, 0.0003, 0.0001], changing at 800, 1200, 1500 epochs. Each experiment is repeated for 10 times and the outliers that the training loss does not decrease are eliminated and we report the average loss results. All the following experiments are conducted on a single GPU of Nvidia GeForce RTX 3090. Codes for our experiments are available at https://github.com/liubc17/eDCNN\_zero\_padding.

There are mainly six purposes in our experiments. The first one is to show the advantage of the  expansive convolutional structure in \eqref{convolution} over the  inner structure in DFCN; The second one is to show the outperformance of eDCNN in extracting the translation-equivalence features, compared with DFCN and cDCNN; The third one  is to  evaluate the  ability of eDCNN in learning clean data; The fourth one is to study the performance of eDCNN in learning noisy data; The fifth one is to show the universal consistency of eDCNN;
The final purpose is to verify the feasibility of eDCNN on two real world data for human activity recognition  and  heartbeat classification.

\begin{figure*}[htbp]
\centerline{\includegraphics[width=0.98\textwidth]{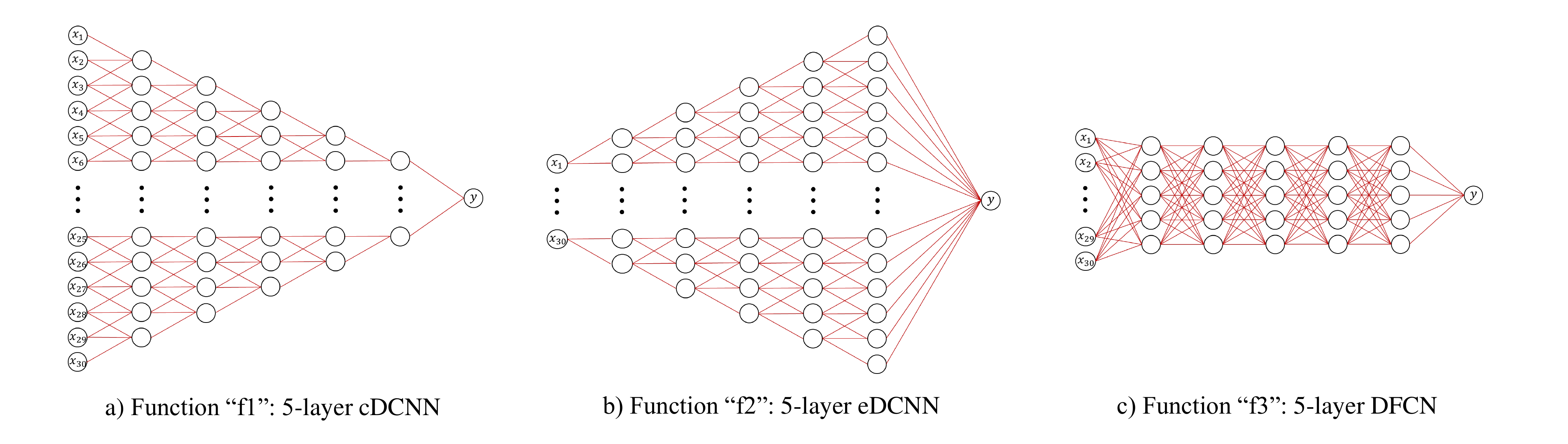}}
\caption{The fitted three functions. ``f1'' is generated by a 5-layer randomly initialized cDCNN. ``f2'' is generated by a 5-layer randomly initialized eDCNN. ``f3'' is generated by a 5-layer randomly initialized fully connected network.}
\label{3 fit functions}
\end{figure*}

\subsection{Advantages of the expansive convolutional structure}

In this simulation, we verify the expansive convolutional structure via comparing it with the inner structure in DFCN and contracting convolutional structure in cDCNN. Our basic idea is to show that the expansive convolutional structure is capable of representing the inner product structure with much fewer free parameters  but not vice-verse. For this purpose, we generate three implicit functions, denoted as ``f1'', ``f2'' and ``f3'' respectively, in Figure \ref{3 fit functions}, where``f1'' is generated by a 5-layer randomly initialized cDCNN,  ``f2'' is generated by a 5-layer randomly initialized eDCNN and ``f3'' is generated by a 5-layer randomly initialized fully connected network. The only difference between ``f1'' and ``f2'' is that ``f2'' pads 2 zeros on both side in the convolution process. For training each function, 900 training samples and 100 test samples are generated. Each sample is a 30-dimension vector with continuous 5-dimension entries uniformly distributed in [0, 1) and zero for the others.

\begin{table}[]
\centering
\caption{The fitting results of 3 functions.}
\label{Tab:fit}
\begin{tabular}{cccc}
\hline
Fit function                        & Network config.             & Test Loss    & Params.       \\ \hline
\multirow{6}{*}{function ``f1''}    & 1-layer fc                  & 5.96e-3      & 300           \\
                                    & 1-block multi-conv cDCNN    & 5.67e-3      & 16             \\
                                    & 1-block multi-conv eDCNN    & 2.23e-3      & 16             \\
                                    & 2-layer fc                  & 7.74e-3      & 400            \\
                                    & 2-block multi-conv cDCNN    & 2.95e-3      & 32             \\
                                    & 2-block multi-conv eDCNN    & 2.40e-3      & 32             \\ \hline
\multirow{6}{*}{function ``f2''}    & 1-layer fc                 & 1.37e-2       & 300            \\
                                    & 1-block multi-conv cDCNN   & 8.95e-3       & 16             \\
                                    & 1-block multi-conv eDCNN   & 8.66e-3       & 16             \\
                                    & 2-layer fc                 & 7.06e-3       & 400            \\
                                    & 2-block multi-conv cDCNN   & 2.89e-3       & 32             \\
                                    & 2-block multi-conv eDCNN   & 2.29e-3       & 32             \\
                                    \hline
\multirow{6}{*}{function ``f3''}    & 1-layer fc                 & 1.04e-3       & 300            \\
                                    & 1-block multi-conv cDCNN   & 8.30e-3       & 16             \\
                                    & 1-block multi-conv eDCNN   & 1.18e-3       & 16             \\
                                    & 2-layer fc                 & 1.16e-3       & 400            \\
                                    & 2-block multi-conv cDCNN   & 7.53e-3       & 32             \\
                                    & 2-block multi-conv eDCNN   & 9.30e-4       & 32             \\ \hline
\end{tabular}
\end{table}



We use multi-level convolutional structures and fully connected structures  to fit the three functions.
As shown in Table \ref{Tab:fit}, we use 1-layer and 2-layer fully connected (denoted as ``fc'') networks as baselines. We use 1-block and 2-block multi-level expansive convolutional (denoted as ``multi-conv'') structures to factorize corresponding 1-layer and 2-layer fully connected networks.
Each  layer of the used 1-layer and 2-layer fully connected networks has 10 units and uses ReLU activation and each block of the used 1-block and 2-block multi-level cDCNN or eDCNN convolutional structures has 5 convolution layers. Each layer has 1 filter of filter size 3. The former 4 layers do not use bias and activation. Only after each block, there is a bias and activation. To match the feature dimension of multi-level eDCNN, a max pooling with pooling size 4 is used after the first block and a max pooling with pooling size 2 is used after the second block.

There are four interesting phenomenon can be found in Table \ref{Tab:fit}: 1)  The multi-level expansive  convolutional structure performs at least not worse than the inner product structure in approximating all three functions with much fewer parameters, showing that presenting the expansive convolutional structure to take place of the inner product structure succeeds in reducing the number of parameters without sacrificing the performance of feature extraction;  2) In fitting ``f1'' and ``f2'' which possesses some translation-equivalence, multi-level convolutional structures achieve much lower test loss than their inner product structure counterparts, exhibiting the outperformance of the convolutional structure in extracting the translation-equivalence; 3) In fitting ``f3'', although 1-block multi-level convolutional structure has slightly higher test loss than 1-layer``fc'', 2-block multi-level convolutional structure achieves the least loss. This is  mainly due to the fact that ``f3'' is generated by fully connected networks and fitting ``f3'' is   easier for the inner product structure. This phenomena illustrates that with more hidden layers, inner product can be represented by the multi-convolutional structure, just as Proposition \ref{Proposition:C-Struc-fea-deep} says; 4) The expansive  convolutional structure performs always better than the contracting convolutional structure, especially for 1-block setting. The reason for this is that the contracting nature limits the representing performance of the corresponding network, although it may encodes some invariance into the structure.
All these demonstrates the power of the expansive convolutional structure and verifies our theoretical assertions in Section \ref{Sec.zero-padding} and Section \ref{Sec.pooling}.

\subsection{eDCNNs  for translation-equivalence data}
\label{translation-invariance}

In this simulation, we
 verify the outperformance of eDCNNs   in approximating translation-equivalence functions over DFCNs and cDCNNs.
 The used translation-equivalence functions are adapted from ``f1'' and ``f2'' in Figure \ref{3 fit functions}, denoted as ``f1m'' and ``f2m''. The only difference is that ``f1m'' and ``f2m'' use a same weight bias with value 0.01 after each convolutional layer. We use fully connected networks, cDCNN and eDCNN to fit ``f1m'' and ``f2m''.
The depth of the used fully connected networks, cDCNN and eDCNN for fitting ``f1m'' and ``f2m'' is set as 4. Each fully connected layer has 10 units and uses ReLU activation. Each convolutional layer of cDCNN and eDCNN has 1 filter of filter size 3 and eDCNN pads 2 zeros on both side in the convolution process. 900 training samples and 100 test samples are generated for our purpose. Each sample is a 30-dimension vector with continuous 5-dimension entries uniformly distributed in [0, 1) and zero for the others.

The fitting RMSE loss results of the three network configurations are shown in Figure \ref{5layer_fitting} and Figure \ref{5layer_edge}, where the ``Network'' axis indicates the network configuration, the``Pooling-Bias'' axis indicates the use of pooling and bias, ``W'' denotes using pooling, ``W/o'' denotes not using pooling, and ``T'', ``F'' and ``S'' denote trainable, non-trainable and restricting the trainable bias vector to be the same, respectively. The only difference between Figure \ref{5layer_fitting} and Figure \ref{5layer_edge} is that we restrict the continuous 5-dimension entries on the edge of the 30-dimension vector for the 100 test samples to show the translation-equivalence of eDCNN.

\begin{figure*}[htbp]
\centerline{\includegraphics[width=\textwidth]{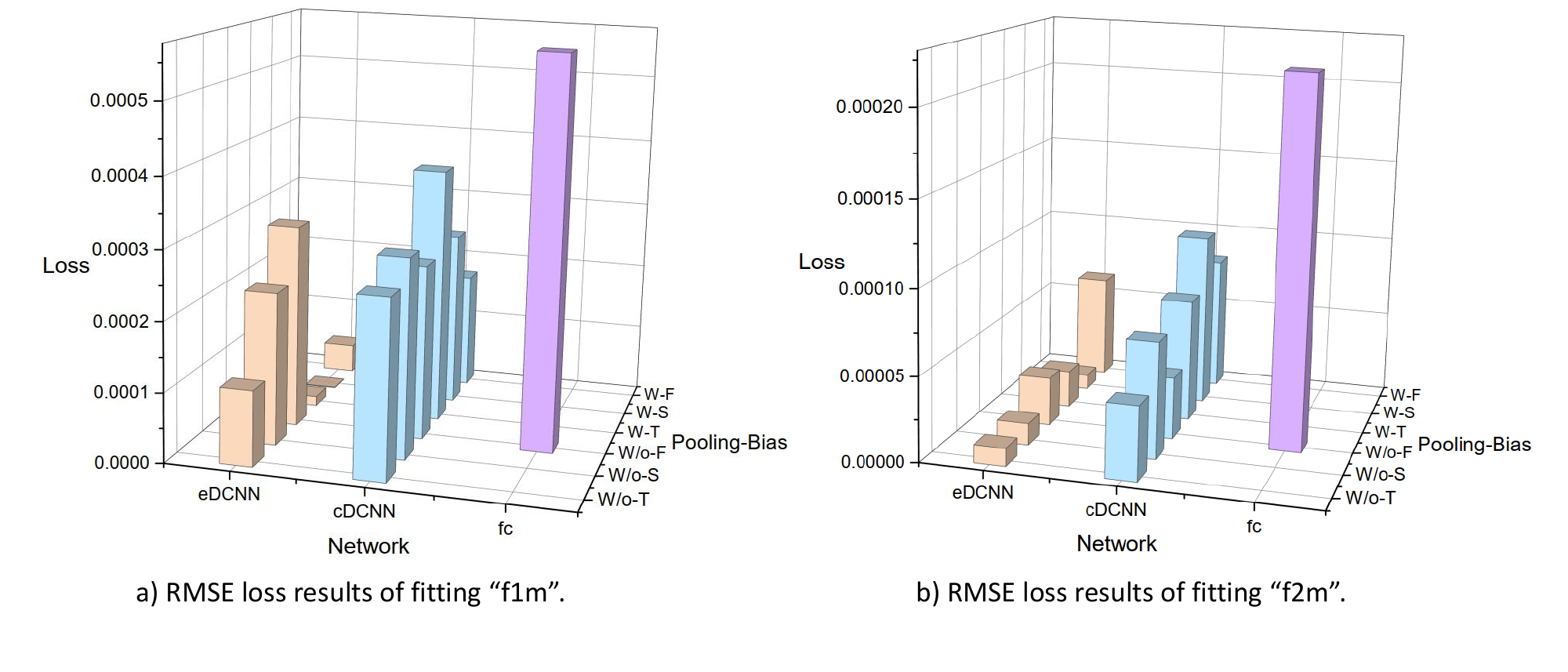}}
\caption{Fitting ``f1m'' and ``f2m'' of DFCN, cDCNN and eDCNN}
\label{5layer_fitting}
\end{figure*}
\begin{figure*}[htbp]
\centerline{\includegraphics[width=\textwidth]{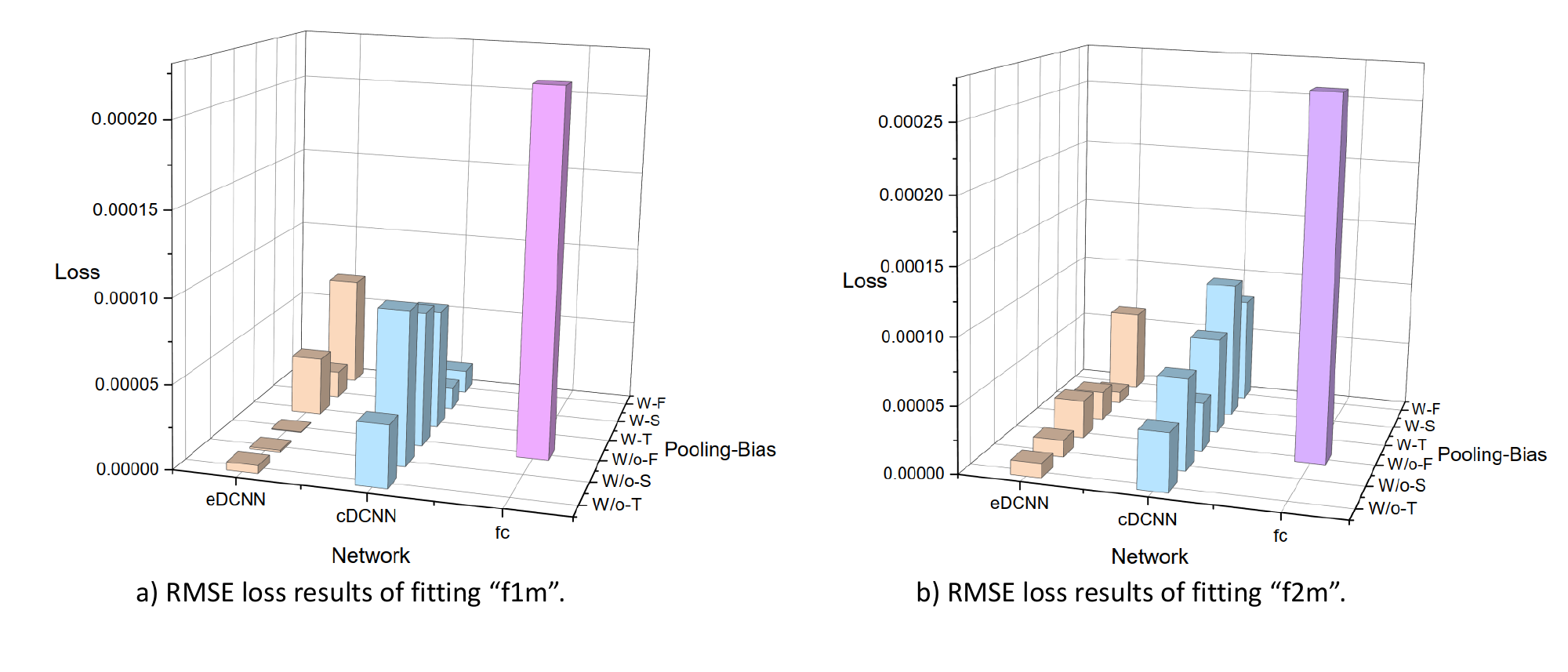}}
\caption{Fitting ``f1m'' and ``f2m'' of DFCN, cDCNN and eDCNN on the position restricted test samples. }
\label{5layer_edge}
\end{figure*}

There are mainly five observations presented in the above figures: 1) In approximating translation-equivalence functions, both cDCNN and eDCNN perform much better than DFCN. The main reason is that cDCNN and eDCNN encode the translation-equivalence in the network structure\footnote{If the support of input is far away from the edge, then cDCNN also possesses the translation-equivalence.} which is beyond the capability of DFCN; 2)  From Figure  \ref{5layer_fitting} where  testing points are drawn  far from the edge, eDCNN performs at  least not worse than cDCNN while for  Figure \ref{5layer_edge} where  testing points are drawn on the edge, eDCNN is much better than cDCNN, showing the outperformance of eDCNN over cDCNN   in encoding the translation-equivalence and verifying Lemma \ref{lemma:bad-translation-equi} and Proposition \ref{Proposition:translation-equiv}; 3) eDCNN with pooling and restricting  trainable bias vectors performs stable for all four tasks. The main reason is that  eDCNN with restricting  trainable bias vectors does not destroy the translation-equivalence of the network and therefore is more suitable for approximating translation-equivalence functions and ``f1m'' and ``f2m'' are biased with  0.01 after each convolutional layer, which verifies Theorem \ref{Theorem:comp-app-dfcn}; 4) Though the effect of pooling is a little bit unstable for different  learning tasks and different bias vector selections, it is efficient for the proposed eDCNN with restricting  trainable bias vectors, which also verifies Theorem \ref{Theorem:comp-app-dfcn}; 5) As far as the bias vector selection is concerned, it can be found in the above figures that
restricting the bias vectors to be constants in the same hidden layers generally outperform other two methods, exhibiting a ``less is more'' phenomena in the sense that less tunable parameters may lead to better approximation performance. This is due to the fact that parameterizing all elements of the bias vector destroys the translation-equivalence of eDCNN. All these show the advantages of the proposed eDCNN and verify our theoretical assertions in Section \ref{Sec.feature-extraction}.

\subsection{Learning ability of eDCNN for clean data}
In  this section, we show the learning performance of eDCNN in learning clean data $y_i=f(x_i)$ for $f$ being either $f_1(x)=x_1 + x_2 + x_3  x_4 + x_5^2$ with the entries of $x$ are uniformly distributed in $(-1, 1)$ or  $f_2(x)=x_1x_2x_3x_4x_5 + x_2x_3x_4x_5x_6 + \cdots + x_{26}x_{27}x_{28}x_{29}x_{30}$ with continuous 5-dimension entries uniformly distributed in $[0, 1)$ and zero for the others. It is easy to check that $f_2$ is translation-invariant while $f_1$ does not have any transformation-invariance property.   For each problem, 1000 training samples and 100 test samples are generated randomly.

To show the power of eDCNN, we consider  the following 5 network configurations:

 1) fully connected networks (denoted as DFCN). Each fully connected layer has 10 units and uses ReLU activation.

2) contracted DCNN (denoted as cDCNN). Each convolutional layer has 1 filter with filter size 3. Each layer has a trainable bias restricted to share the same value, expect for the last layer has a trainable bias without any restriction. Each layer does not use padding and is followed by a ReLU activation.

3) contracted DCNN followed by one fully connected layer (denoted as ``cDCNN+fc''). The setting is the same as cDCNN configuration, except that there is a fully connected layer after the convolution module. The number of units of the fully connected layer is the same as the output dimension of the convolution module after flatten operation.

4) expansive DCNN (denoted as eDCNN). Each convolutional layer has 1 filter with filter size 3. Each layer has a trainable bias restricted to share the same value, expect for the last layer has a trainable bias without any restriction. Each layer pads 2 zeros on both side in the convolution process and is followed by a ReLU activation.

5) expansive DCNN followed by one pooling layer (denoted as ``eDCNN+pl''). The setting is the same as eDCNN configuration, except that there is a max pooling with pooling size 2 after the convolution module.

 \begin{figure}[htbp]
\centerline{\includegraphics[width=0.48\textwidth]{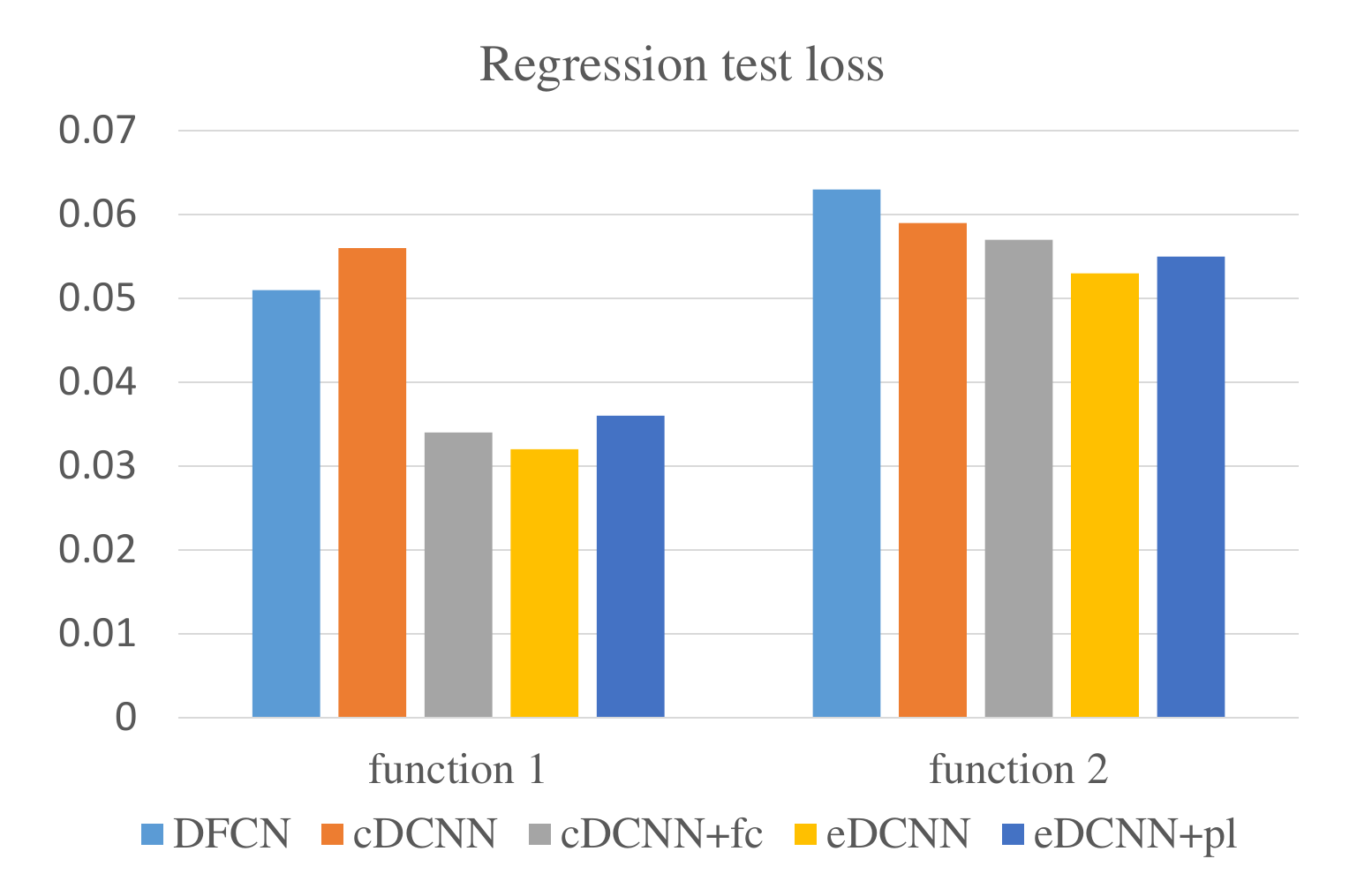}}
\caption{Learning performance of mentioned approaches on clean data}
\label{regression}
\end{figure}

The numerical results can be found in Figure \ref{regression}. There are three interesting observation showed in the figure: 1)   eDCNN as well as ``eDCNN+pl'' always performs better than other two structures in both translation-invariance data and non-translation-invariance data. Due to the universality in approximation and learning, eDCNN is capable of learning arbitrary regression functions neglecting whether it possesses some transformation-invariance, which is far beyond the capability of cDCNN; 2) cDCNN performs worse than DFCN in learning $f_1$ while better than DFCN in learning $f_2$, although cDCNN does not possesses the universality in approximation and learning. This is due to the fact that $f_2$ possesses translation-invariance  and the convolutional structure is easy to capture this invariance while the classical inner product structure fails; 3) eDCNN
performs a little bit than the popular ``cDCNN+fc'' in both learning tasks, showing that eDCNN would be a preferable alternative of the classical setting of convolutional neural networks. All these verify our theoretical setting in Section \ref{Sec.Learning} for clean data.

\begin{table}[]
\centering
\caption{The learning results of 5 network architectures.}
\label{Tab:learning}
\resizebox{\linewidth}{!}{
\begin{tabular}{c|ccccc}
\hline
Test Position & \multicolumn{5}{c}{beginning}                  \\
Network       & DFCN   & cDCNN  & cDCNN+fc & eDCNN  & eDCNN+pl \\
Loss          & 0.0578 & 0.0496 & 0.0447   & 0.0401 & 0.0406   \\ \hline
Test Position & \multicolumn{5}{c}{middle}                     \\
Network       & DFCN   & cDCNN  & cDCNN+fc & eDCNN  & eDCNN+pl \\
Loss          & 0.0597 & 0.0430 & 0.0467   & 0.0401 & 0.0406   \\ \hline
Test Position & \multicolumn{5}{c}{end}                        \\
Network       & DFCN   & cDCNN  & cDCNN+fc & eDCNN  & eDCNN+pl \\
Loss          & 0.0698 & 0.0556 & 0.0426   & 0.0385 & 0.0415   \\ \hline
\end{tabular}
}
\end{table}

\begin{figure}[htbp]
\centerline{\includegraphics[width=0.4\textwidth]{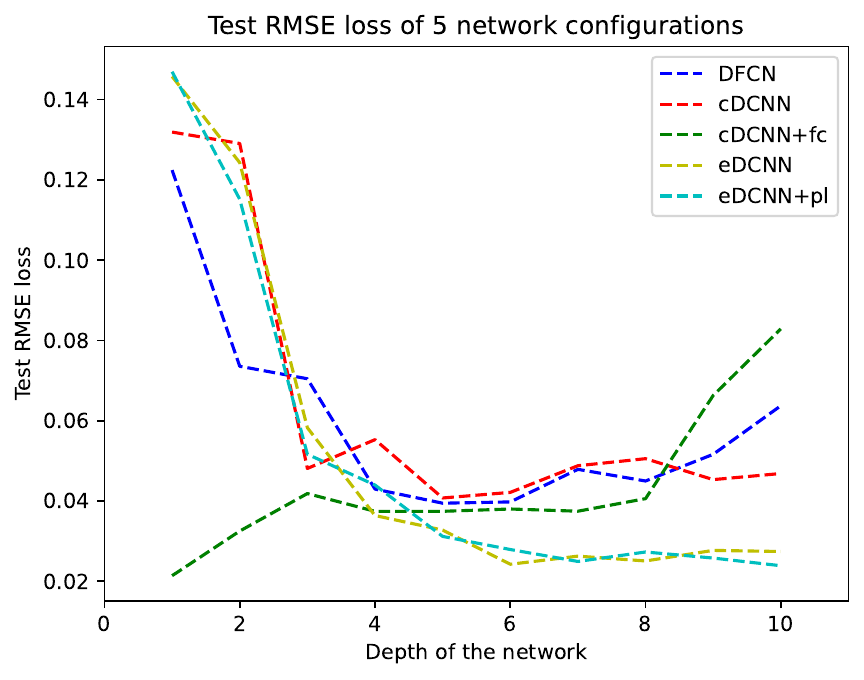}}
\caption{Relation between RMSE  of test position ``beginning'' and network depth. The number of training samples is 20000.}
\label{f3_first_samples20000}
\end{figure}

\subsection{Learning ability of eDCNN for noisy data}
In this simulation, we show the outperformance of eDCNN in learning noisy data. The data are generated by $y_i=f(x_i)+\sigma_i$ with $f_3(x) = \sin{(x_1^2 + \cdots + x_{30}^2)} + \frac{1}{2}\cos{(x_1^2 + \cdots + x_{30}^2)}$  with only continuous 5-dimension entries uniformly distributed in [0, 1) and zero for the other and $\sigma_i$ is drawn i.i.d. according to the Gaussian distribution $\mathcal N(0,0.01)$. To show the differece between cDCNN and eDCNN in reflecting the translation-invariance, we evaluate on different positions of the continuous 5-dimension entries. Our numerical results can be found in   Table \ref{Tab:learning},  where test position ``beginning'' denotes that the continuous 5-dimension entries are at positions 1 to 5, test position ``middle'' denotes that the continuous 5-dimension entries are at positions 13 to 17, and test position ``end'' denotes that the continuous 5-dimension entries are at positions 26 to 30. We also figure out the role of depth for the mentioned five structures in Figure \ref{f3_first_samples20000}.

From Table \ref{Tab:learning} and Figure \ref{f3_first_samples20000}, it is safe to draw the following two conclusions: 1)  eDCNN as well as ``eDCNN+pl'' performs better than DFCN and cDCNN in three test positions. Especially, for ``beginning'' and ``end'' positions, eDCNN performs much better. This is due to the fact that cDCNN does not encode the translation-equivalence when the support of function is at the edge, showing the necessity of zero-padding; 2) eDCNN as well as ``eDCNN+pl'' performs at least not worse than ``cDCNN+fc''. Furthermore,  it follows from Figure \ref{f3_first_samples20000} that  eDCNN and ``eDCNN+pl'' behave stable  with respect to the depth, which is similar as ``cDCNN+fc''. Both observations
illustrate that zero-padding is a feasible and efficient alternative of fully connected layer in learning noisy data.



\begin{figure*}[t]
    \centering
	\subfigure[Relation for $f_2$]{\includegraphics[width=7cm,height=6cm]{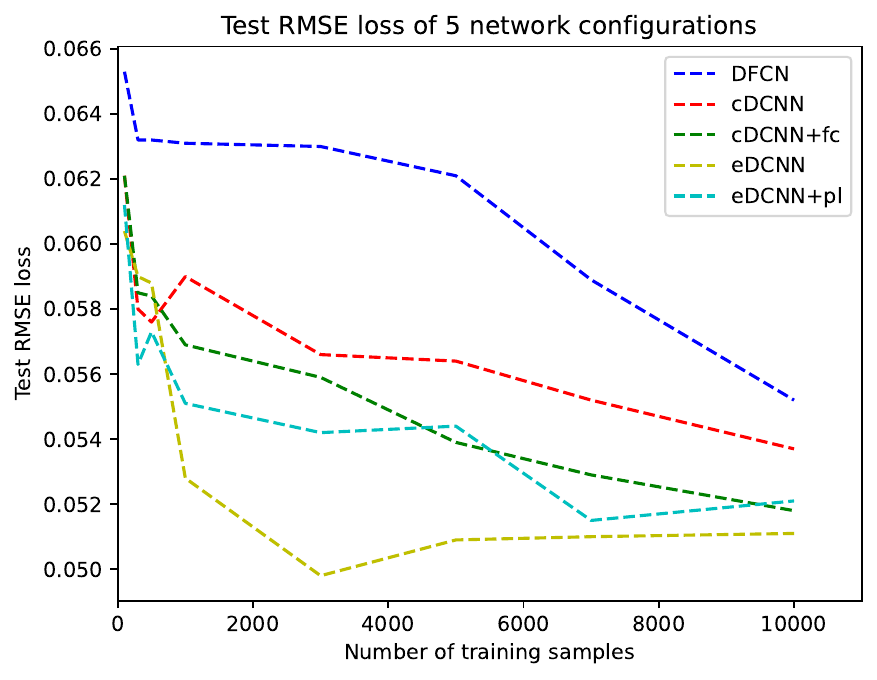}}\hspace{0.3in}
    \subfigure[Relation for $f_3$]{\includegraphics[width=7cm,height=6cm]{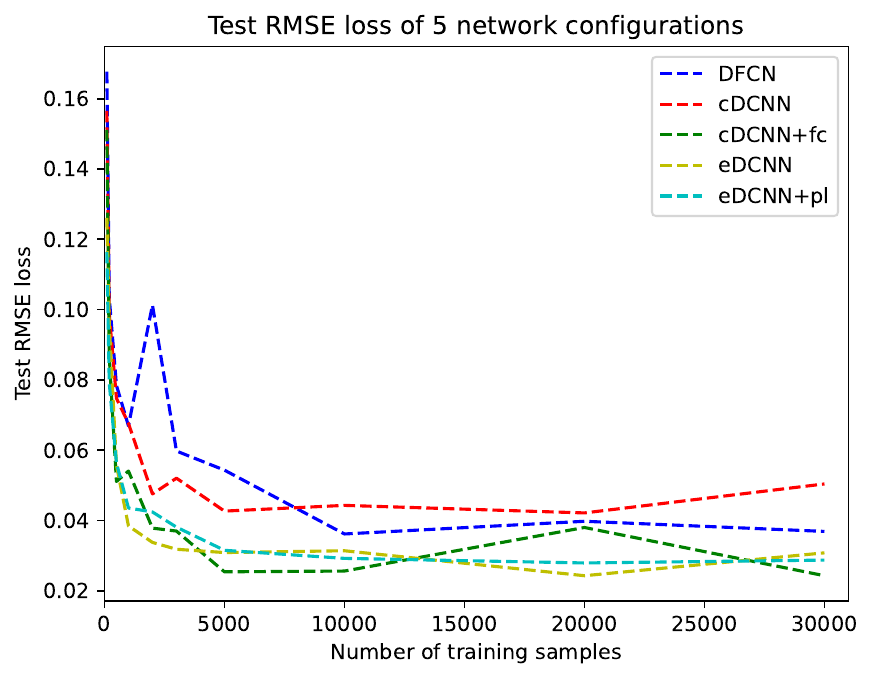}}
	\caption{Relation between RMSE and data size for $f_2$ and $f_3(x)$ with test position ``beginning''.}
\label{Fig:ComparisonMSEtest}
\end{figure*}

\subsection{Universal consistency of eDCNN}
In this simulation, we show the universal consistency of eDCNN via showing the relation between the  generalization error  and size of the data. In particular, we report this relation on both clean and noisy data for the mentioned 5  network configurations. To be detailed, we report the smallest loss among different network depths for learning $f_2$ and set depth to be 6 for learning $f_3$. The reason of different setting for different simulations is due to the stability of depth shown in the previous subsection.
The numerical results can be  found in Figure \ref{Fig:ComparisonMSEtest}.

There are also three interesting observations exhibited in Figure \ref{Fig:ComparisonMSEtest}: 1) For $f_2$, the test loss curve of  eDCNN  is much lower than others while ``eDCNN+pl''                                   achieves the suboptimal results, which validates that eDCNN has better learning capability than DFCNs and cDCNNs in learning clean translation-invariant data; 2) For $f_3$, eDCNN, ``eDCNN+pl'' and ``cDCNN+fc'' are better than DFCN and cDCNN, mainly due to the non-translation-invariance of DFCN and cDCNN when  $f_3$ is supported at the edge; 3)  RMSE roughly  decreases with respect to the number of training samples, demonstrating the consistency of the deep nets estimates. All these verify our theoretical assertions in Section \ref{Sec.Learning}.



\begin{figure}[htbp]
\centerline{\includegraphics[width=0.4\textwidth]{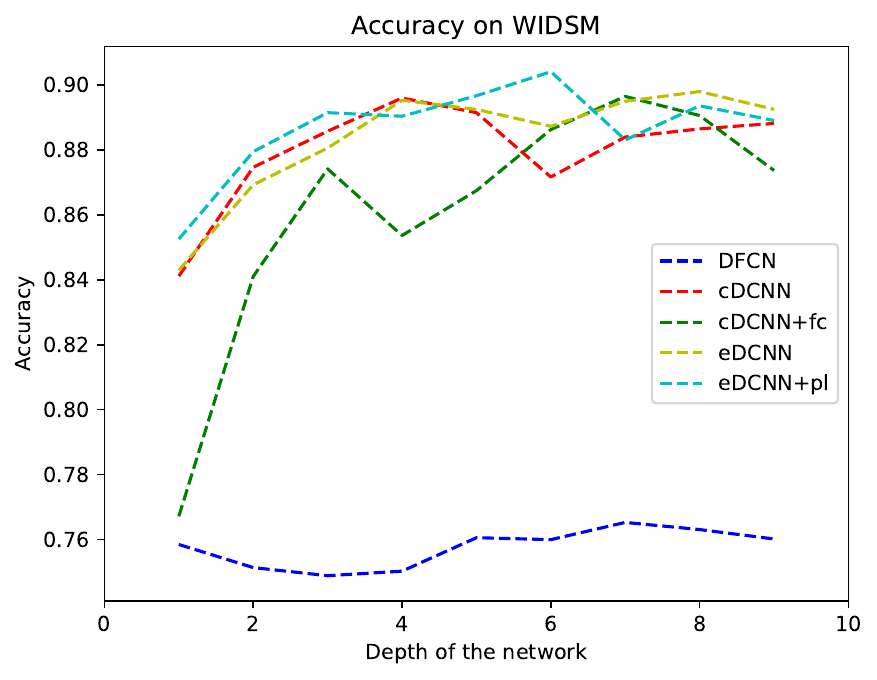}}
\caption{Classification accuracy results of the 5 network configurations on WISDM dataset.}
\label{WISDM}
\end{figure}



\begin{figure}[htbp]
\centerline{\includegraphics[width=0.4\textwidth]{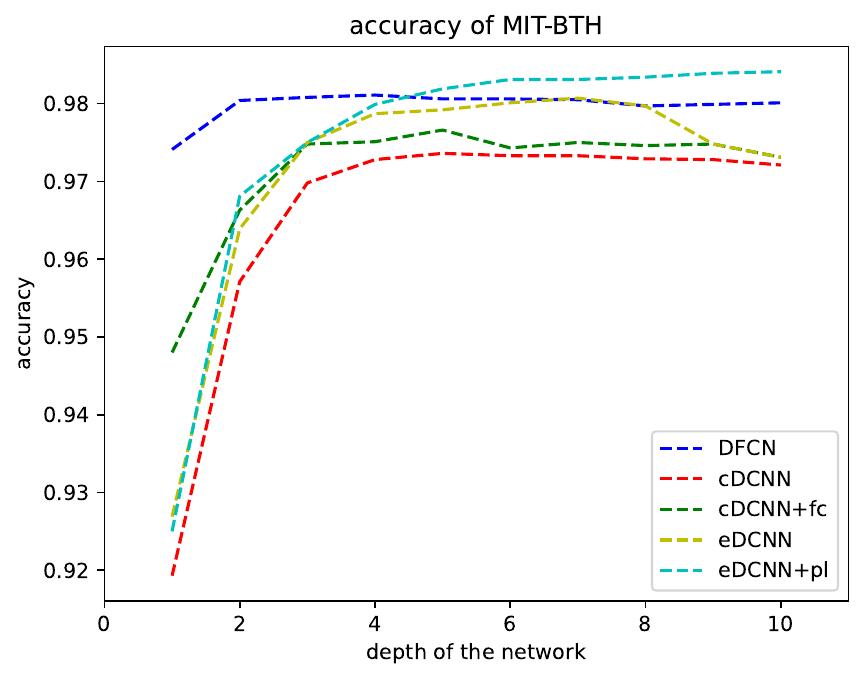}}
\caption{Classification accuracy results of the 5 network configurations on MIT-BTH dataset.}
\label{MITacc}
\end{figure}

\begin{figure}[htbp]
\centerline{\includegraphics[width=0.4\textwidth]{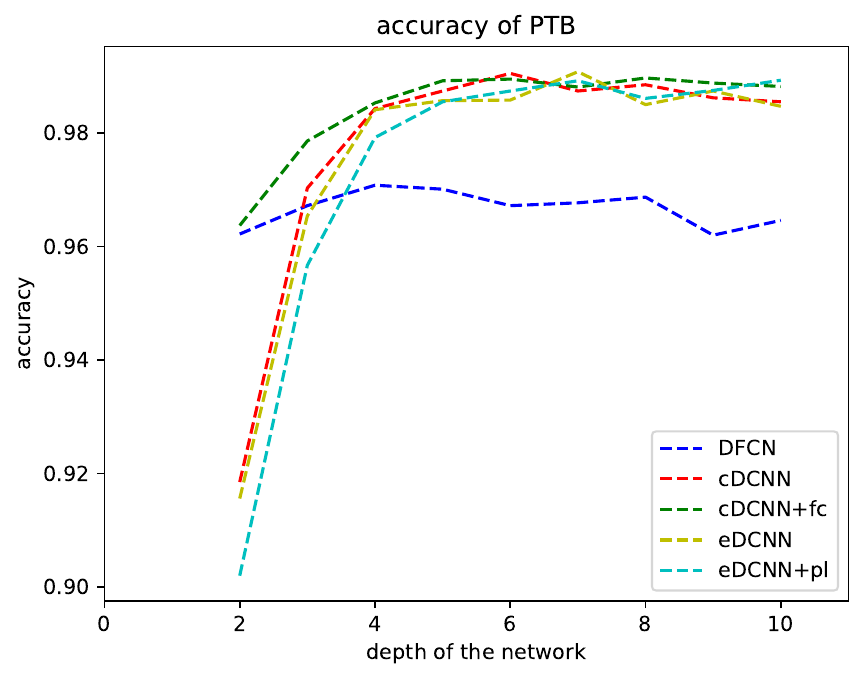}}
\caption{Classification accuracy results of the 5 network configurations on PTB dataset.}
\label{PTBacc}
\end{figure}

\subsection{Real date everifications}
In this part, we aim at showing the usability and efficiency of eDCNN on two real world data.
\subsubsection{WISDM dataset}
For human activity recognition task, we evaluate on WISDM dataset which includes 1098207 samples with 6 categories of movement (walking, running, jogging, etc.).

The network structures are as follows. There are  60 units in each  fully connected layer and 20 filters with filter length 9 in the the convolutional layer of cDCNN, ``cDCNN+fc'', eDCNN and ``eDCNN+pl''.
The number of units of the fully connected layer of ``cDCNN+fc'' is set as 100. The pooling size of the max pooling layer of ``eDCNN+pl'' is set as 2 and the max pooling layer lies after the convolutional module. We train the network with an Adam optimizer for 150 epochs. The batch size is set as 512.

Figure \ref{WISDM} shows the classification accuracy results of the 5 network configurations on WISDM dataset varying with depth. It can be found that ``eDCNN+pl'' achieves the highest test accuracy and holds high test accuracy with different depths, which shows the good learning capacity of eDCNN with appropriate pooling scheme. DFCN is not good at dealing with such data that contains temporal information. Even when dealing with such data, cDCNN outperforms ``cDCNN+fc'' overall, which further validates the disadvantage of the fully connected layer in dealing with such data.


\subsubsection{ECG heartbeat dataset}

For ECG heartbeat classification task, we evaluate on the MIT-BIH Arrhythmia Database and the PTB Diagnostic ECG Database. An ECG is a 1D signal that is the result of recording the electrical activity of the heart using an electrode. It is a very useful tool that cardiologists use to diagnose heart anomalies and diseases. The MIT-BIH Arrhythmia data set includes 109446 samples with 5 categories, and the PTB Diagnostic ECG Database includes 14552 samples with 2 categories.  Similarly, the network depth of DFCN, cDCNN, ``cDCNN+fc'', eDCNN and ``eDCNN+pl'' varies from 1 to 10. For fully connected layers, each layer has 80 units. For the convolutional layer, each layer has 16 filters with length  6. The number of units of the fully connected layer is   64. The pool size of the max pooling layer of ``eDCNN+pl'' is  2 and the max pooling layer lies after the first convolutional layer. We train the network with an Adam optimizer for 100 epochs.

Figure \ref{MITacc} shows the classification accuracy results of the 5 network configurations on MIT-BTH dataset varying with depth. ``eDCNN+pl'' achieves the highest test accuracy and holds high test accuracy with large network depths, which shows the good learning capacity of eDCNN with appropriate pooling scheme. DFCN achieves good and stable results among different network depths. The relatively poor results of cDCNN and ``cDCNN+fc'' may due to the information loss while convolving at the edge of feature without bias.

Figure \ref{PTBacc} shows the classification accuracy results of the 5 network configurations on PTB dataset varying with depth. DFCN is not good at dealing with this dataset. The other four configurations achieve good and vary close results. It may because this dataset is relatively easy to classify.

 \section{Conclusion}
In this paper, we studied the representation and learning performance of deep convolutional neural networks with zero-padding (eDCNN). After detailed analysis of the roles of zero-padding, pooling and bias vectors, we found that eDCNN succeeds in encoding the translation-equivalence (or translation-invariance for eDCNN with pooling) without sacrificing the universality in approximation and learning of DFCNs. This demonstrated that eDCNN is essentially better than  DFCN  in translation-equivalence related applications, since the network structure itself can reflect this data feature without any training. Noting that deep neural networks without zero-padding (cDCNN) do not possess the universality, we found that eDCNN is essentially better than cDCNN in the sense that eDCNN can extract any  data features if it is sufficiently deep, which is beyond the capability of cDCNN. This assertion was also made in extracting the  translation-equivalence (or translation-invariance) by terms that cDCNN fails to encode it if the support of data is on the edge.   All these findings together with solid theoretical analysis and numerous numerical experiments
illustrated the outperformance of eDCNN over cDCNN and DFCN and provided a unified network structure with theoretical verifications for feature extraction and learning purposes.

\section*{Acknowledgement}
Z. Han was partially supported by the National Key Research
and Development Program of China under Grant 2020YFB1313400, the CAS Project for Young Scientists in Basic Research under Grant YSBR-041, the National Natural Science Foundation of China under Grant 61903358, 61773367, 61821005, and the Youth Innovation Promotion Association of the Chinese Academy of Sciences under Grant 2022196, Y202051.
S. B. Lin was partially
	supported by   the
	National Natural Science Foundation of China [Grant Nos. 62276209]. D. X. Zhou was partially  supported in part by the Research Grants Council of Hong Kong [Project Nos. CityU 11308020, N\_CityU 102/20, C1013-21GF], Hong Kong Institute for Data Science,  Germany/Hong Kong Joint Research Scheme  [Project No. G-CityU101/20],
  Laboratory for AI-Powered Financial Technologies, and National Science Foundation of China [Project No. 12061160462].

\bibliographystyle{IEEEtran}
\bibliography{DCNN}{}

\clearpage

\section*{Appendix: Proofs}
We prove our theoretical results in Appendix. The proof idea  of  Lemma \ref{lemma:communication} is standard by either using the definition of the convolution or  using the the symbol of sequences in \cite{daubechies1992ten,zhou2020universality}.

\begin{proof}[Proof of Lemma \ref{lemma:communication}]
Due to \eqref{convolution1} and the fact that  $\vec{w}^1=(w_j^1)_{j=-\infty}^{\infty}$
and $\vec{w}^2=(w_j^2)_{j=-\infty}^{\infty}$ are supported on $\{0,1,\dots,s\}$, for any $-s\leq k\leq 0, k\leq j\leq s+k$, the $j$-th component of $\vec{w}^1\star\vec{w}^2$ is
\begin{eqnarray*}
   && (\vec{w}^1\star\vec{w}^2)_j=\sum_{k=j-s}^{j}
    w^1_{j-k}w^2_{k+s}\\
    &=&
    w_0^1w_{j+s}^2+w_1^1w_{j+s-1}^2+\dots+w_{j+s-1}^1w_1^2+w^1_{j+s}w_0^2.
\end{eqnarray*}
Noting again that $\vec{w}^1$  and $\vec{w}^2$ have the same support, we thus have
$$
  (\vec{w}^1\star\vec{w}^2)_j=(\vec{w}^2\star\vec{w}^1)_j,
$$
which implies \eqref{culmulative-law} for $\otimes=\star$. To prove \eqref{culmulative-law} for $\otimes=*$, we use the
symbol of sequences as follows.
The symbol  $\tilde{\vec{w}}$ of a sequence  $\vec{w}$ supported on the set of nonnegative integers is defined as a polynomial on $\mathbb C$ by $\tilde{\vec{w}}=\sum_{k=0}^\infty w_kz^k$. Then, direct computation yields
$$
\widetilde{\vec{w}^1*\vec{w}^2}=\tilde{\vec{w}}^1\tilde{\vec{w}}^2.
$$
Noting
$$
  \tilde{\vec{w}}^1\tilde{\vec{w}}^2=\tilde{\vec{w}}^2\tilde{\vec{w}}^1=\widetilde{\vec{w}^2*\vec{w}^1},
$$
we then have
$$
   \widetilde{\vec{w}^1*\vec{w}^2}=\widetilde{\vec{w}^2*\vec{w}^1}
$$
and consequently
$$
    \vec{w}^1*\vec{w}^2 = \vec{w}^2*\vec{w}^1.
$$
This completes the proof of Lemma \ref{lemma:communication}.
\end{proof}

 Define $T_{\tilde{d},d'}^{\vec{w}}$ be a $\tilde{d}\times d'$ Toeplitz matrix    by
\begin{equation}\label{convolution-matrix}
         \left[\begin{array}{ccccccc}
                w_0 & 0 & 0 & \cdots & 0\\
                w_1 & w_0& 0 & \cdots & 0\\
                \vdots &  \ddots &  \ddots & \ddots &\vdots\\
                w_{d'-1} & w_{d'-2} & \cdots&\cdots&w_0\\
                w_{d'}& w_{d'-1}& \cdots&  0\cdots & w_1\\
                 \vdots & \ddots & \ddots  & \ddots &\vdots\\
                 w_{\tilde{d}-d'}& \cdots&\cdots& \cdots& w_{\tilde{d}-2d'+1}\\
                  \vdots & \ddots &   \ddots & \ddots &\vdots\\
                    w_{\tilde{d}-2}&w_{\tilde{d}-3} &\cdots& w_{\tilde{d}-d'}&w_{\tilde{d}-d'-1}\\
                  w_{\tilde{d}-1}&w_{\tilde{d}-2}&\cdots&w_{\tilde{d}-d'+1}&w_{\tilde{d}-d'}
                  \end{array}\right].
\end{equation}
 Observe that if $W$ is supported in $\{0,\dots,S\}$ for some $S\in\mathbb N$, the entry $(T_{D,d'}^{\Vec{w}})_{i,k}=w_{i-k}$ vanishes when $i-k>S$. To prove Proposition \ref{Proposition:translation-equiv}, it is sufficient to prove the following lemma.

\begin{lemma}\label{Lemma:translation-equivalence}
Let $1\leq p\leq d$, $2\leq s\leq d$, $d_0=d$ and $d_\ell=d+\ell s$. If $\vec{w}$ is supported on $\{0,1,\dots,s\}$ and $T_{d_\ell,d_{\ell-1}}^{\vec w}$ is the Toeplitz matrix in (\ref{convolution-matrix}). Then  for any $j\leq d_{\ell-1}-p+1$, there holds
$$
    T_{d_\ell,d_{\ell-1}}^{\vec w} A_{j,d_{\ell-1}} \vec{v}_{p,d_{\ell-1},1}
    =
     A_{j,d_{\ell}}T_{d_\ell,d_{\ell-1}}^{\vec w} \vec{v}_{p,d_{\ell-1},1}.
$$
\end{lemma}

\begin{proof}
Direct computation yields
\begin{equation}\label{convolution-and-Toep}
    \vec{w}^\ell*\vec{v}=T_{d_\ell,d_{\ell-1}}^{\vec{w}^\ell}\vec{v}.
\end{equation}
Due to \eqref{role-of-transi}, we have
$$
     A_{j,d_{\ell-1}} \vec{v}_{p,d_{\ell-1},1}=\vec{v}_{p,d_{\ell-1},j}.
$$
Then, for any $k=1,\dots,d_\ell$, we have from  \eqref{convolution-and-Toep} and \eqref{convolution} that
\begin{eqnarray*}
  &&\left(T_{d_\ell,d_{\ell-1}}^{\vec w} A_{j,d_{\ell-1}} \vec{v}_{p,d_{\ell-1},1}\right)_k
   =
  (\vec{w}* \vec{v}_{p,d_{\ell-1},j})_k\\
  &=&
  \sum_{i=j}^{j+p-1}w_{k-i}v_{i-j+1}
\end{eqnarray*}
and
$$
\left(T_{d_\ell,d_{\ell-1}}^{\vec w}   \vec{v}_{p,d_{\ell-1},1}\right)_k =(\vec{w}* \vec{v}_{p,d_{\ell-1},1})_k
=
\sum_{i=1}^pw_{k-i}v_{i}.
$$
According to \eqref{role-of-transi} again, we have for any $k=j,\dots,j+p-1$,
\begin{eqnarray*}
  && \left(A_{j,d_{\ell}}T_{d_\ell,d_{\ell-1}}^{\vec w}   \vec{v}_{p,d_{\ell-1},1}\right)_k=
   \left(T_{d_\ell,d_{\ell-1}}^{\vec w}   \vec{v}_{p,d_{\ell-1},1}\right)_{k-j+1}\\
 &=&
   \sum_{i=1}^pw_{k-j+1-i}v_{i}=\sum_{i=j}^{j+p-1}w_{k-i}v_{i-j+1}.
\end{eqnarray*}
Noting further that except for the $j,\dots,j+p-1$ components, other components are zero. We then have
$$
    T_{d_\ell,d_{\ell-1}}^{\vec w} A_{j,d_{\ell-1}} \vec{v}_{p,d_{\ell-1},1}
    =
     A_{j,d_{\ell}}T_{d_\ell,d_{\ell-1}}^{\vec w} \vec{v}_{p,d_{\ell-1},1}.
$$
This completes the proof of Lemma \ref{Lemma:translation-equivalence}.
\end{proof}

Based on Lemma \ref{Lemma:translation-equivalence}, we can prove Proposition \ref{Proposition:translation-equiv} directly.
\begin{proof}[Proof of Proposition \ref{Proposition:translation-equiv}]
Due to Lemma \ref{Lemma:translation-equivalence}, we have
\begin{eqnarray*}
      && T_{d_\ell,d_{\ell-1}}^{\vec{w}^\ell}\cdots T_{d_1,d}^{\vec{w}^1} A_{j,d} \vec{v}_{p,d,1}
    =
     T_{d_\ell,d_{\ell-1}}^{\vec{w}^\ell}\cdots A_{j,d_{1}}T_{d_1,d_{0}}^{\vec{w}^1} \vec{v}_{p,d,1}\\
     &=&
     \dots=
     A_{j,d_{\ell}}T_{d_\ell,d_{\ell-1}}^{\vec{w}^\ell}\cdots T_{d_1,d_{0}}^{\vec{w}^1} \vec{v}_{p,d,1}
\end{eqnarray*}
Then it follows from  (\ref{convolution-and-Toep}) that \eqref{trans-equiv}
holds.
This completes the proof of Proposition \ref{Proposition:translation-equiv}.
\end{proof}

To prove Proposition \ref{Proposition:C-Struc-fea-deep}, we need several auxiliary lemmas. The first one is the
 convolution   factorization lemma  provided in  \cite[Theorem 3]{zhou2020universality}.

\begin{lemma}\label{Lemma:convolution-fraction}
Let $S\geq 0,2\leq s\leq d$ and $\vec{u}=(u_k)_{-\infty}^\infty$ be supported on $\{0,\dots,S\}$. Then there exists $L<\frac{S}{s-1}+1$ filter vectors $\{\vec{w}^{\ell}\}_{\ell=1}^L$ supported on $\{0,\dots,s\}$ such that $\vec{u}=\vec{w}^{L}*\cdots*\vec{w}^{1}$.
\end{lemma}

The second one  is the  following lemma that was given in  \cite{zhou2018deep}.

\begin{lemma}\label{Lemma:convolution-to-matrix}
 Let $2\leq s\leq d'$, $d'_\ell=d'+\ell s$ and $d'_0=d'$. If $\{\vec{w}^{\ell}\}_{\ell=1}^{L^*}$ is supported on  $\{0,\dots,s\}$, then
\begin{equation}\label{matrix-fraction}
    T_{d_\ell',d'}^{\vec{w}^\ell*\cdots \vec{w}^2*\vec{w}^1}=T_{d'_\ell,d'_{\ell-1}}^{\vec{w}^\ell}\cdots T_{d_2',d'_1}^{\vec{w}^2}T_{d'_1,d'}^{\vec{w}^1}
\end{equation}
holds for any $\ell\in\{1,2,\dots,L^*\}$.
\end{lemma}

The third one focuses on representing the inner product by multi-layer convolution.

\begin{lemma}\label{Lemma:factorization}
Let $\tilde{d},d'\in\mathbb N$, $2\leq s\leq d'$, $d'=0$ and $d'_\ell=d'+\ell s$. Then for any $\tilde{d}\times d'$ matrix $W$, there exist   $L^*=\lceil\frac{d'\tilde{d}}{s-1}\rceil$ filter vectors $\{\vec{w}^\ell\}_{\ell=1}^{L^*}$ supported on $\{0,1,\dots,s\}$ such that
\begin{equation}\label{Covo-app1}
     W x= \mathcal S_{d'_{L^*},d',0}\left(\vec{w}^{L^*} *\dots*\vec{w}^{1}*x\right),\qquad \forall x\in\mathbb R^{d'}.
\end{equation}
\end{lemma}

\begin{proof}
Define a sequence  $\vec{u}$ supported on $\{0,\dots,d'\tilde{d}-1\}$ by stacking the rows of $W$, i.e.
\begin{eqnarray*}
    \vec{u}^T
    &=&(W_{1,1},W_{2,1},\dots,W_{d\tilde{d},1},W_{1,2},\dots,W_{\tilde{d},d})\\
    &=:&
    (W_{0},\dots,W_{d'\tilde{d}-1}).
\end{eqnarray*}
  We apply Lemma \ref{Lemma:convolution-fraction} with $S=d'\tilde{d}-1$ and $s\geq 2$ to obtain   $\hat{L}<\frac{d'\tilde{d}-1}{s-1}+1$ filter vectors $\{\vec{w}^\ell\}_{\ell=1}^{\hat{L}}$ such that
$$
    \vec{u}= \vec{w}^{\hat{L}}*\dots*\vec{w}^1.
$$
Let $T^{\vec{u}}$ be the $d'_{\hat{L}}\times d'$ matrix   $(W_{k-j})_{k=1,\dots,d_{\hat{L}}',j=1,\dots,d'}$. Observe from the definition of the sequence $\vec{u}$ that for $j=1,\dots,\tilde{d}$, the $jd'$-th row of $T^{\vec{u}}$ is exactly the $j$-th raw of $W$.   Setting  $L^*=\lceil\frac{d'\tilde{d}}{s-1}\rceil$, it is obvious that $\hat{L}\leq L^*$.  Taking $\vec{w}^{\hat{L}+1}=\dots=\vec{w}^{L^*}$ to be the delta sequence, we have $\vec{u}=\vec{w}^{L^*}*\dots*\vec{w}^1$. Then it follows from Lemma \ref{Lemma:convolution-to-matrix} that
$
    T^{\vec{u}}=T^{\vec{w}^{L^*}}\cdots T^{\vec{w}^1}.
$
According to the definition of $\mathcal S_{d_{L^*}',d',0}$ in (\ref{def.down-sampling})  and the definition of $T^{\vec{u}}$, we have
$$
    \mathcal S_{d_{L^*}',d',0}\left(T^{\vec{u}}x\right)=W x.
$$
Therefore,
\begin{eqnarray*}
   &&\mathcal S_{d_{L^*}',d',0} (T^{\vec{w}^{L^*}}\dots      T^{\vec{w}^1}x)\\
& =&
  \mathcal S_{d_{L^*}',d',0} (\vec{w}^{L^*}*\dots*\vec{w}^1*x)
   = Wx.
\end{eqnarray*}
This completes the proof of Lemma \ref{Lemma:factorization}.
\end{proof}

Our fourth lemma is a representation of matrices multiplication, which can be derived from Lemma  \ref{Lemma:factorization}.

\begin{lemma}\label{Lemma:Multi-factorization}
Let $J,d\in\mathbb N$, $2\leq s\leq d$, $d_1,\dots,d_J\in\mathbb N$ satisfying $d_0=d$ and  $L_j=\lceil\frac{d_{j-1}d_j}{s-1}\rceil$.  If $W^j$  is a $d_j\times d_{j-1}$ matrix, then  for any $x\in\mathbb R^d$,
there exist $\sum_{j=1}^JL_j$   filter vectors $\{\vec{w}^{j,\ell}\}_{\ell=1}^{L_j}$ supported on $\{0,1,\dots,s\}$ such that
\begin{eqnarray}\label{Covo-app2}
    &&W^J W^{J-1} \cdots W^1 x\nonumber\\
    &=&
    \mathcal S_{d_{J-1}+L_Js,d_{J-1},0}\left(\vec{w}^{J,L_J} *\cdots*\vec{w}^{J,1}\right)\nonumber\\
    &&*\dots*
    \mathcal S_{d_{0}+L_1s,d_0,0}\left(\vec{w}^{1,L_1} *\dots*\vec{w}^{1,1}*x\right).
\end{eqnarray}
\end{lemma}

With these helps, we are in a position to prove Proposition \ref{Proposition:C-Struc-fea-deep}

\begin{proof}[Proof of Proposition \ref{Proposition:C-Struc-fea-deep}]
According to Proposition \ref{Proposition:translation-equiv}, for any $1\leq j,j'$, the $j$-th component of
$\vec{w}^L*\dots *\vec{w}^1*\vec{v}_{p,d,j}$ is exactly the $j'$-th component of $\vec{w}^L*\dots *\vec{w}^1*\vec{v}_{p,d,j'}$. We then have from  (\ref{def.down-sampling}) that   \eqref{trans-invar} holds. Then Proposition \ref{Proposition:C-Struc-fea-deep} is a direct result by combining   Lemma \ref{Lemma:Multi-factorization} and \eqref{trans-invar}. This completes the proof of Proposition \ref{Proposition:C-Struc-fea-deep}.
\end{proof}


To prove Theorem \ref{Theorem:app-dcnn-res}, we need some preliminaries. For a sequence $\vec{w}$ supported on $\{0,1,\dots,s\}$, write $\|\vec{w}\|_1=\sum_{k=-\infty}^\infty|w_k|$ and $\|\vec{w}\|_\infty=\max_{-\infty\leq k\leq \infty}|w_k|$. Define $B^0:=\max_{x\in\mathbb I^d}\max_{k=1,\dots,d}|x^{(k)}|$ and
$$
    B^\ell:=\|\vec{w}^\ell\|_1B^{\ell-1}\cdots B^1B^0,\qquad   \ell\geq 1.
$$
Let $T_{d_{\ell},d_{\ell-1}}^{\vec{w}^{\ell}}$ be the convolution matrix given in \eqref{convolution-matrix}. Then for any $j=1,\dots,d_\ell$, direct computation yields
\begin{equation}\label{bound-1}
     \max_{x\in\mathbb I^d}\left|\left(T_{d_{\ell},d_{\ell-1}}^{\vec{w}^{\ell}}\dots
     T_{d_{1},d_{0}}^{\vec{w}^{1}} x\right)_j\right|\leq B^\ell
\end{equation}
and for $1\leq k\leq \ell-1$
\begin{equation}\label{bound-1.1}
      \left|\left(T_{d_{\ell},d_{\ell-1}}^{\vec{w}^{\ell}}\dots
     T_{d_{k+1},d_{k}}^{\vec{w}^{k+1}}  B^k\mathbf{1}_{d_k}\right)_j\right|\leq B^\ell.
\end{equation}
\begin{lemma}\label{Lemma:Induction}
Let $\ell\in\mathbb N$, $2\leq s\leq d$ and $\mathcal C^R_{\ell,\vec{w}^\ell,b^\ell}$ be defined by \eqref{convolutional-mapping-res} with $\vec{w}^\ell$ supported on $\{0,1,\dots,s\}$ and $b^\ell=2^{\ell-1}B^\ell$, then
\begin{eqnarray}\label{induction}
 &&\sigma\circ \mathcal C^R_{\ell,\vec{w}^\ell, {b}^\ell} \circ \sigma  \circ \dots \circ \sigma\circ\mathcal C^R_{1,\vec{w}^1, {b}^1}(x)\nonumber\\
 &=&
 T_{d_\ell,d_{\ell-1}}^{\vec{w}^\ell}\cdots T_{d_2,d_1}^{\vec{w}^2}T_{d_1,d}^{\vec{w}^1}x +b^\ell{\bf 1}_{d_\ell}\nonumber \\
 &+&
 \sum_{k=1}^{\ell-1}T_{d_\ell,d_{\ell-1}}^{\vec{w}^\ell}\cdots T_{d_{k+1},d_{k}}^{\vec{w}^{k+1}}
 b^k{\mathbf 1}_{d_k}.
\end{eqnarray}
\end{lemma}

\begin{proof} For any $1\leq \ell \leq L^*$, it follows  from (\ref{convolution-and-Toep}) that
\begin{equation}\label{convolutio-to-matrices}
     \vec{w}^{\ell} *\dots*\vec{w}^{1}*x=
     T_{d_{\ell*},d_{\ell-1}}^{\vec{w}^{\ell}}\dots
     T_{d_{1},d_{0}}^{\vec{w}^{1}} x.
\end{equation}
We then prove \eqref{induction} by induction.
We obtain from the definition of $\sigma$ directly that
$$
    \sigma(\vec{w}^1*x+b^1 {\bf 1}_{d_1})
    =
    T_{d_{1},d_{0}}^{\vec{w}^{1}} x+ b^1 {\bf 1}_{d_1},
$$
which verifies \eqref{induction} for $\ell=1$.
Assume that \eqref{induction} holds for $\ell-1$, that is,
\begin{eqnarray*}
 &&\vec{V}_{\ell-1}^{s,R}=\sigma\circ \mathcal C^R_{\ell-1,\vec{w}^{\ell-1}, {b}^{\ell-1}} \circ \sigma  \circ \dots \circ \sigma\circ\mathcal C^R_{1,\vec{w}^1, {b}^1}(x)\nonumber\\
 &=&
 T_{d_{\ell-1},d_{\ell-2}}^{\vec{w}^{\ell-1}}\cdots T_{d_2,d_1}^{\vec{w}^2}T_{d_1,d}^{\vec{w}^1}x+b^{\ell-1}{\bf 1}_{d_{\ell-1}}\\
 &+&
 \sum_{k=1}^{\ell-2}T_{d_{\ell-1},d_{\ell-2}}^{\vec{w}^{\ell-1}}\cdots T_{d_{k+1},d_{k}}^{\vec{w}^{k+1}}
 b^k{\mathbf 1}_{d_k}.
\end{eqnarray*}
Then,
\begin{eqnarray*}
    && \vec{w}^\ell*\vec{V}^{s,R}_{\ell-1}
     =
   T_{d_{\ell},d_{\ell-1}}^{\vec{w}^{\ell}}\cdots T_{d_1,d}^{\vec{w}^1}x
   \nonumber\\
   & +&
 \sum_{k=1}^{\ell-1}T_{d_{\ell},d_{\ell-1}}^{\vec{w}^{\ell}}T_{d_{\ell-1},d_{\ell-2}}^{\vec{w}^{\ell-1}}\cdots T_{d_{k+1},d_{k}}^{\vec{w}^{k+1}}
 b^k{\mathbf 1}_{d_k}.
\end{eqnarray*}
This together with \eqref{bound-1}, \eqref{bound-1.1} and $b^\ell=2^{\ell-1}B^\ell$ yields that for any $j=1,\dots,d_\ell$, there holds
$$
    \max_{x\in\mathbb I^d}\left|\left(\vec{w}^\ell*\vec{V}^{s,R}_{\ell-1} \right)_j\right|\leq B^\ell+B^\ell\sum_{k=1}^{\ell-1} 2^{k-1}=2^{\ell-1}B^\ell=b^\ell,
$$
implying
\begin{eqnarray*}
   &&\sigma\left(\vec{w}^\ell*\vec{V}^{s,R}_{\ell-1} +b^\ell{\bf 1}_{d_\ell}\right)
   =\vec{w}^\ell*\vec{V}^{s,R}_{\ell-1} +b^\ell{\bf 1}_{d_\ell}\\
   &=&
   T_{d_{\ell},d_{\ell-1}}^{\vec{w}^{\ell}}\cdots T_{d_1,d}^{\vec{w}^1}x\\
   &+&
   \sum_{k=1}^{\ell-1}T_{d_{\ell},d_{\ell-1}}^{\vec{w}^{\ell}}T_{d_{\ell-1},d_{\ell-2}}^{\vec{w}^{\ell-1}}\cdots T_{d_{k+1},d_{k}}^{\vec{w}^{k}}
 b^k{\mathbf 1}_{d_k}
 +b^\ell{\bf 1}_{d_\ell}.
\end{eqnarray*}
Therefore, \eqref{induction} holds for any $\ell=1,2,\dots$.
This completes the proof of Lemma \ref{Lemma:Induction}.
\end{proof}

By the help of the above lemma, we are in a position to prove Theorem \ref{Theorem:app-dcnn-res} as follows.

\begin{proof}[Proof of Theorem \ref{Theorem:app-dcnn-res}]
Define
\begin{equation}\label{def.b}
    b^\ell=2^{\ell-1}B^\ell,\qquad \ell=1,2,\dots,L^*.
\end{equation}
It follows from Lemma \ref{Lemma:Induction} with $\ell=L^*-$ and \eqref{convolution-and-Toep} that
\begin{eqnarray*}
 &&\vec{V}_{L^*-1}^{s,R}
  =
 \vec{w}^{L^*-1}* \cdots * \vec{w}^{1} *x +b^{L^*-1}{\bf 1}_{d_{L^*-1}}\nonumber \\
 &+&
 \sum_{k=1}^{L^*-2}T_{d_{L^*-1},d_{L^*-2}}^{\vec{w}^{L^*-1}}\cdots T_{d_{k+1},d_{k}}^{\vec{w}^{k+1}}
 b^k{\mathbf 1}_{d_k}.
\end{eqnarray*}
Write
\begin{eqnarray*}
  && \vec{B}^{d_{L^*}}
    :=
   \sum_{k=1}^{L^*}T_{d_{L^*},
   d_{L^*-1}}^{\vec{w}^{L^*}}\cdots T_{d_{k+1},d_{k}}^{\vec{w}^{k+1}}
   b^k{\mathbf 1}_{d_k}\\
   &=&
    w^{L^*}*
   b^{L^*-1}{\bf 1}_{d_{L^*-1}} \\
  &+ &
  w^{L^*}*\sum_{k=1}^{L^*-2}T_{d_{L^*-1},d_{L^*-2}}^{\vec{w}^{L^*-1}}\cdots T_{d_{k+1},d_{k}}^{\vec{w}^{k+1}}
 b^k{\mathbf 1}_{d_k}.
\end{eqnarray*}
For $\vec{\theta}:=(\theta_1,\dots,\theta_{n})^T$, define
$
      \vec{\theta}_{d_{L^*}}\in\mathbb R^{d_{L^*}}$ be the vector satisfying $\mathcal S_{d_{L^*},d,0}(\theta_{d_{L^*}})=\theta$ and  other components being zero.
Then, we can define
$$
  \vec{b}^{L^*}:= \theta_{d_{L^*}}- \vec{B}^{d_{L^*}}.
$$
Therefore, it follows from the definition of $\mathcal S_{d_{L^*},d,0}$  and Lemma \ref{Lemma:factorization}  that
\begin{eqnarray*}
 &&\mathcal S_{d_{L^*},d,0} \circ\sigma (\vec{w}^{L^*}* \vec{V}_{L^*-1}^{s,R}+ \vec{b}^{L^*})\\
  &=&
 \sigma  \circ \mathcal S_{d_{L^*},d,0} (\vec{w}^{L^*} * \vec{V}_{L^*-1}^{s,R}+ \vec{b}^{L^*})\\
 &=&
 \sigma  \circ \mathcal S_{d_{L^*},d,0} (\vec{w}^{L^*} * \vec{V}_{L^*-1}^{s,R}+\theta_{d_{L^*}}- \vec{B}^{d_{L^*}})\\
 &=&
  \sigma  \circ \mathcal S_{d_{L^*},d,0} ( \vec{w}^{L^*}* \cdots * \vec{w}^{1} *x
  +\theta_{d_{L^*}})\\
  &=&
  \sigma ( \mathcal S_{d_{L^*},d,0} \circ \vec{w}^{L^*}* \cdots * \vec{w}^{1} *x
  +\theta)\\
  &=&\sigma (Wx+\theta).
\end{eqnarray*}
This proves \eqref{eDCNN-for-shallow}. Furthermore \eqref{trans-inv-eDCNN-1} can be derived by Corollary \ref{Corollary:trans-invar} by noting that the bias vectors in the same layer are the same before  pooling.
This completes the proof of Theorem \ref{Theorem:app-dcnn-res}.
\end{proof}

\begin{proof}[Proof of Theorem \ref{Theorem:comp-app-dfcn}]
Since $\mathcal S_{d_{L_\ell}+1,d_{L_\ell},0}$ maps $\mathbb R^{d_{L_\ell}+1}$ to $\mathbb R^{d^*_\ell}$, the we can ues \eqref{eDCNN-for-shallow} $L$ times and obtain \eqref{edcnn-dfcn-app} directly.  The translation invariance \eqref{trans-inv-edcnn-for-dfcn} can also be derived from \eqref{trans-inv-eDCNN-1} directly. This completes the proof
 of Theorem \ref{Theorem:comp-app-dfcn}.
\end{proof}

The proof of Theorem \ref{Theorem:universal consistency} is almost the same as that in \cite{lin2022universal}. The only difference is that the pooling scheme reduces the covering number of eDCNN. It should be highlighted that this proof is semi-trivial
 and we sketch it for the sake of completeness.

 At first, we introduce the definition of the $\ell^1$ empirical covering number.
For  a set of functions $\mathcal V$, denote by
$\mathcal N_1(\epsilon,\mathcal V)$ the $\ell^1$ empirical $\varepsilon$ covering number \cite[Def. 9.3]{gyorfi2002distribution} of $\mathcal V$   which is
 the number of
elements in a least $\varepsilon$-net of $\mathcal V$ with respect to the metric $\|f\|_{D,1}:=\frac1{|D|}\sum_{i=1}^{|D|}|f(x_i)|$.  Define
$$
   \pi_M\mathcal H^{s,R,d^*_{L-1},\dots,d^*_0,0}_{L_L,\dots,L_1}:=\left\{\pi_Mf:f\in\mathcal H^{s,R,d^*_{L-1},\dots,d^*_0,0}_{L_L,\dots,L_1}\right\}.
$$
It is easy to check that deep nets in $H^{s,R,d^*_{L-1},\dots,d^*_0,0}_{L_L,\dots,L_1}$ is of depth at most
$$
    L_{pooling}:=\frac{\sum_{j=1}^Ld_{j-1}^*d_j^*}{s-1}+2L
$$
and with free parameter at most
\begin{eqnarray*}
    n_{pooling}&:= &
    L_{pooling}s+\sum_{j=1}^Ld^*_L\\
    &=&
   \left(\frac{\sum_{j=1}^Ld_{j-1}^*d_j^*}{s-1}+2L\right)s+\sum_{j=1}^Ld^*_L.
\end{eqnarray*}

The following lemma that can be derived by using the same approach as that in \cite[Lemma 4]{lin2022universal} describes the covering number of $\pi_M\mathcal H^{s,R,d^*_{L-1},\dots,d^*_0,0}_{L_L,\dots,L_1}$.

\begin{lemma}\label{Lemma:covering-number}
Let $2\leq s\leq d$ and $L\in\mathbb N$.
For any $0<\varepsilon\leq M$, if $d_1^*,\dots,d_L^*\geq d+1$, then there holds
\begin{eqnarray*}
    &&\log_2\sup_{D\in\mathcal Z^m}\mathcal N_1(\epsilon,\pi_M\mathcal H^{s,R,d^*_{L-1},\dots,d^*_0,0}_{L_L,\dots,L_1})\\
   &\leq & c^* \left(\frac{\sum_{j=1}^Ld_{j-1}^*d_j^*}{s-1}+2L\right)
   \log  (Ld_{\max}^*)\\
   &&\left(\left(\frac{\sum_{j=1}^Ld_{j-1}^*d_j^*}{s-1}+2L\right)s+ \sum_{j=1}^Ld^*_L\right)\log\frac{M}\epsilon,
\end{eqnarray*}
where $c^*$ is an absolute constant and $d^*_{\max}=\max_{1\leq j\leq L}d^*_j.$
\end{lemma}

We then use the above covering number estimate to bound generalization error for bounded samples. Write $y_M=\pi_My$ and $y_{i,M}=\pi_My_i$. Define
$$
   \mathcal E_{\pi_M}(f)=\int_{\mathcal Z}(f(x)-y_M)^2d\rho,
$$
and
$$
     \mathcal E_{\pi_M,D}(f)=\frac1m\sum_{i=1}^m(f(x_i)-y_{i,M})^2.
$$
The following lemma that can be derived directly from \cite[Theorem 11.4]{gyorfi2002distribution} and Lemma \ref{Lemma:covering-number}
shows the performance of eDCNN with pooling for learning bounded samples.

\begin{lemma}\label{Lemma:error-bounded}
If $M^2_{D}|D|^{-\theta}\rightarrow0$ and (\ref{condition of  universal}) holds for some $\theta\in(0,1/2)$, then
$$
     \mathcal E_{\pi_M}(\pi_Mf^{s,R,d^*_{L-1},\dots,d^*_0,0}_{D,L_L,\dots,L_1})-\mathcal E_{\pi_M,D}(\pi_Mf^{s,R,d^*_{L-1},\dots,d^*_0,0}_{D,L_L,\dots,L_1})\rightarrow 0
$$
holds almost surely.
\end{lemma}

We then sketch the proof of Theorem \ref{Theorem:universal consistency} as follows.

\begin{proof}[Proof of Theorem \ref{Theorem:universal consistency}]
Based on the above lemmas,  Theorem \ref{Theorem:universal consistency} can be derived as the standard approach in the proof of \cite[Theorem 1]{lin2022universal} by dividing the generalization error into eight items. The only thing difference is that to guarantee the universal approximation property of  $f^{s,R,d^*_{L-1},\dots,d^*_0,0}_{D,L_L,\dots,L_1}$, it follows from Theorem  \ref{Theorem:comp-app-dfcn}   that  if $d_1^*,\dots,d_\ell^*\geq d+1$, then
  for any $\varepsilon>0$, there exists some
 $g_\varepsilon\in \mathcal H^{s,R,d^*_{L-1},\dots,d^*_0,0}_{L_L,\dots,L_1}$  with sufficiently large $\sum_{j=1}^Ld_{j-1}^*d_j^*$ such that
$$
            \|f_\rho-g_\varepsilon\|_{L^2 ({\rho_X})}^2\leq
           \varepsilon.
$$
The strongly universal consistency then can be derived by using the same method  as that in proving \cite[Theorem 1]{lin2022universal} directly. Furthermore, the translation-invariance \eqref{trans-inv-universal} can be derived from \eqref{trans-inv-edcnn-for-dfcn}. This completes the proof of Theorem \ref{Theorem:universal consistency}.
\end{proof}



\end{document}